%% file: mahjong.tex
\documentclass[preprint,12pt]{elsarticle}

\usepackage{graphicx}
\usepackage{tikz}

\usepackage{amsmath,amssymb,amsthm}
\usepackage[normalem]{ulem}
\usepackage{url}
\theoremstyle{plain}
\newtheorem{theorem}{Theorem}
\newtheorem{lemma}{Lemma}
\newtheorem{proposition}{Proposition}

\theoremstyle{definition}
\newtheorem{dfn}{Definition}
\newtheorem*{remark}{Remark}
\newtheorem{myexp}{Example}
\newenvironment{myexpcont}
{\addtocounter{myexp}{-1}
\begin{myexp}[cont'd]}
  {\end{myexp}}

\usepackage{algpseudocode}
\usepackage[Algorithmus]{algorithm}

\algnewcommand{\IIf}[1]{\State\algorithmicif\ #1\ \algorithmicthen}
\algnewcommand{\EndIIf}{\unskip\ \algorithmicend\ \algorithmicif}

\newcommand{\define}{\ensuremath{\triangleq}}
\usepackage{xcolor}

\newcommand{\red}[1]{\textcolor{red}{#1}}

\newcommand{\M}{\mathcal{M}}

\newcommand{\dfncy}{{\sf{dfncy}}}

\newcommand{\cost}[1]{{\sf{cost}}(#1)}

\newcommand{\val}{\sf{val}}

\newcommand{\pd}{{\sf{pDCMP}}}
\newcommand{\qd}{{\sf{qDCMP}}}

\journal{arXiv}

\begin{document}

\begin{frontmatter}
\title{A Fast Algorithm for Computing the  Deficiency Number of a Mahjong Hand}

\author[inst1]{Xueqing Yan}
\affiliation[inst1]{organization={School of Computer Science},
            addressline={Shaanxi Normal University}, 
            city={Xi'an},
            country={China}}

\author[inst1]{Yongming Li}
\author[inst2]{Sanjiang Li}
\affiliation[inst2]{organization={Centre for Quantum Software and Information},
            addressline={University of Technology Sydney}, 
            city={Sydney},
            country={Australia}}
            
%
\begin{abstract}
The tile-based multiplayer game Mahjong is widely played in Asia and has also become increasingly popular worldwide. Face-to-face or online, each player begins with a hand of 13 tiles and players draw and discard tiles in turn until they complete a winning hand. An important notion in Mahjong is the \emph{deficiency} number (a.k.a. \emph{shanten} number in Japanese Mahjong)  of a hand, which estimates how many tile changes are necessary to complete the hand into a winning hand. The deficiency number plays an essential role in major decision-making tasks such as selecting a tile to discard. This paper proposes a fast algorithm for computing the deficiency number of a Mahjong hand. Compared with the baseline algorithm, the new algorithm is usually 100 times faster and, more importantly, respects the agent's knowledge about available tiles. The algorithm can be used as a basic procedure in all Mahjong variants by both rule-based and machine learning-based Mahjong AI. 
\end{abstract}



\end{frontmatter}

\section{Introduction}

Games have played as the test-beds of many novel artificial intelligence (AI) techniques and ideas since the very beginning of AI research. In the past decades, we have seen AI programs that can beat best human players in perfect information games including checker \cite{Samuel59}, chess \cite{shannon1988programming} and Go \cite{Silver+16}, where players know everything occurred in the game before making a decision.  More recently, important progress has also been made in solving the more challenging imperfect information games such as the two-player heads-up  Texas hold'em poker \cite{bowling2015heads,BrownS17}, DouDiZhu \cite{Jiang19_deltadou}, and Mahjong 
\cite{li2020suphx}.



In this paper, we study the imperfect information game Mahjong, which is a widely played multiplayer game. Compared with Go and poker, Mahjong is perhaps more popular and more complicated. The game is played with a set of 144 tiles based on Chinese characters and symbols (see Figure~\ref{fig:mjtiles})  and has many variants in tiles used, rules, and scoring systems \cite{Wiki_mahjong}. In the beginning of the game, each player is given with a hand of 13 tiles and, in turn, players draw and discard tiles until they complete a winning hand. 

Most early researches focus on developing AI programs for Japanese Mahjong and are often reported in Japanese. Mizukami and Tsuruoka \cite{MizukamiT15} 
build a strong Mahjong AI based on Monte Carlo simulation and opponent models. Yoshimura et al. \cite{Yoshimura_optmov} propose a tabu search method of optimal movements without using game records. More recently, Kurita and Hoki \cite{Kurita_methods} represent the game of Mahjong as multiple Markov Decision Processes and the thus constructed Mahjong AI was evaluated to be competitive with two strong AI players.  Suphx, the Mahjong AI from Microsoft Research Asia \cite{li2020suphx}, is by far the strongest AI for Japanese Mahjong. Based on deep reinforcement learning, Suphx has demonstrated super-human performance in the Tenhou platform (https://tenhou.net/). 

There are also some very recent researches on Chinese Mahjong. In \cite{WangYLH19}, Wang et al. propose a deep residual network-based strategy that is compared favourably with three high-level human players.
Gao and Li \cite{Gao-bloody} extract Mahjong features with the deep learning model and then derive their strategy by classifying the learned features. More progresses are also reported in the IJCAI 2020 Mahjong AI Competition.\footnote{ \url{https://botzone.org.cn/static/gamecontest2020a.html}}  

While these achievements are remarkable, they (except \cite{Yoshimura_optmov}) often demand high volume game records of human expert players. In addition, like general deep learning algorithms, most of these Mahjong AIs have the limitation of poor interpretability. In a previous paper \cite{li2019lets}, we initiated a mathematical and knowledge-based AI study of Mahjong, aiming to explore ways of integrating traditional knowledge-based AI with machine learning. As a combinatorial game, winning hands of Mahjong have structures that can be exploited and specified as local constraints. Moreover, the agent's belief or knowledge about the distribution of non-shown tiles can be updated as knowledge/belief revision. We expect that knowledge reasoning methods, e.g., Bayesian inference, constraint solving and optimisation, belief revision, can be exploited in the design of strong Mahjong AI. It is also expected that these knowledge-based methods can be further combined with machine learning methods proposed in \cite{bowling2015heads,BrownS17,li2020suphx}. 

To build a strong Mahjong AI, one basic procedure is to determine how many tile changes we need to complete the hand. This notion is called the deficiency number of the hand in \cite{li2019lets}, which is exactly the \emph{shanten} number (or simply shanten) in Japanese Mahjong. Deficiency calculation is a basic procedure that can be used in all Mahjong variants by both rule-based and learning-based Mahjong AIs (see, e.g., \cite{Yoshimura_optmov,Kurita_methods}). While algorithms for calculating deficiency or shanten exist (see Tenhou's website for a shanten calculator and the quadtree method of \cite{li2019lets}), these algorithms do not take the agent's knowledge about available tiles into consideration, which makes their calculated numbers less useful.    

The aim of this paper is to provide an efficient and general purpose algorithm for calculating the deficiency (shanten) number which respects the agent's knowledge of available tiles. Recall that a hand is a winning hand if it can be decomposed into four melds (i.e., triples or three consecutive tiles) plus one pair.\footnote{There are possibly other kinds of winning hands in different Mahjong variants, but this is the most typical one.} We call any such combination a \emph{decomposition}. As in the quadtree method \cite{li2019lets}, we represent the agent's knowledge of available tiles as a $k$-tuple of integers less than or equal to 4, where $k$ is the number of different tiles used in the Mahjong variant. Furthermore, we divide the hand into blocks of tiles, calculate their quasi-decompositions and map them into a small set of types, and then merge these local types and evaluate their costs, where the precise meaning of a quasi-decomposition will become clear in Definition~\ref{dfn:qdcmp}. For the time being, we regard it as an incomplete version of decomposition. Meanwhile, the type of a quasi-decomposition is a 7-tuple which consists of values of seven attributes of the  quasi-decomposition.

The remainder of this paper is structured as follows. After a short introduction of basic notions of Mahjong in Sec.~\ref{sec:background}, we recall in Sec.~\ref{sec:quadtree} the quadtree method for calculating deficiency introduced in \cite{li2019lets} and point out its limitations. The revised knowledge-aware definition of deficiency is introduced in Sec.~\ref{sec:kb-deficiency}, where we also introduce seven attributes of quasi-decompositions. Furthermore, we show that the revised deficiency can be obtained by calculating the minimal cost of all quasi-decompositions and the cost of a quasi-decomposition can be determined, in most time, by its attributes. Our block-based algorithm for calculating deficiency is described in Sec.~\ref{sec:block_dfncy}, where we show that if the deficiency is not greater than four, then our algorithm can return the number exactly. Experiments in Sec.~\ref{sec:experiment} demonstrate that the block-based algorithm is efficient and mostly exact. The last section concludes the paper with outlook on future research directions.  Technical proofs can be found in the appendix.
\section{Backgrounds}\label{sec:background}
In this section, we recall from \cite{li2019lets} some basic notions of Mahjong. For simplicity, we only consider a very basic and essential version of Mahjong. The other variants can be dealt with analogously. There are three types of tiles used in the basic version of Mahjong:
\begin{itemize}
    \item Bamboos: $B1, B2, ..., B9$, each with four identical tiles
    \item Characters: $C1, C2, ..., C9$, each with four identical tiles
    \item Dots: $D1, D2, ..., D9$, each with four identical tiles
\end{itemize}
We call \emph{Bamboo} ($B$), Character ($C$), Dot ($D$) \emph{colours} of the tiles, and write $\M_{0}$ for the set of these tiles, which includes in total 108 tiles. 

\begin{figure}[tbp]
    \centering
    \includegraphics[width=0.8\textwidth]{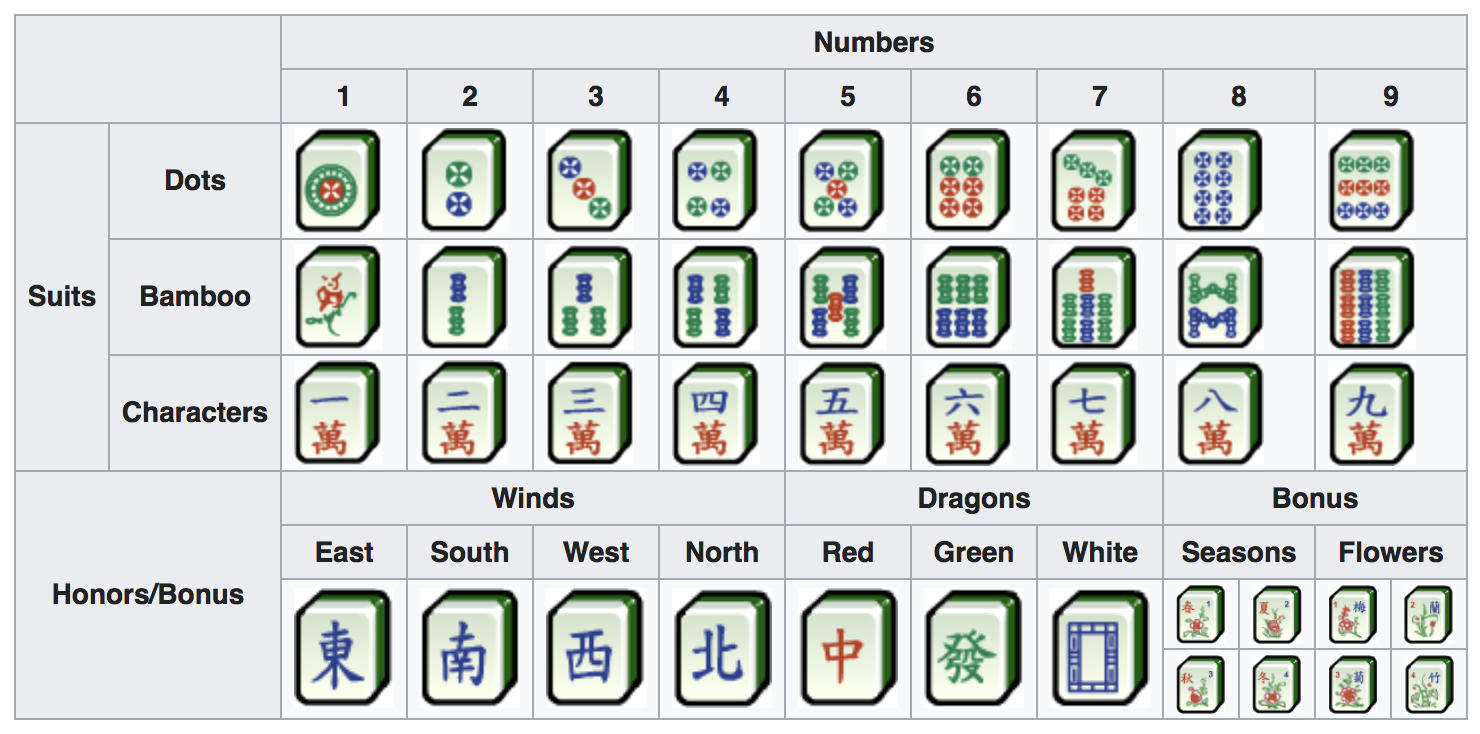}
    \caption{Mahjong tiles, from https://en.wikipedia.org/wiki/Mahjong}\label{fig:mjtiles}
    \label{fig:my_label}
\end{figure}

\begin{dfn}[pong, chow, meld]
A \emph{pong} is a sequence of three identical tiles. A \emph{chow} is a sequence of three consecutive tiles of the same colour. A \emph{meld} is either a pong or a chow. 
\end{dfn}

The following non-standard notions are also useful.
\begin{dfn} [pchow, pmeld \cite{li2019lets}] \label{dfn:pmeld}
A \emph{prechow} {\rm(}\emph{pchow}{\rm)} is a pair of tiles of the same colour s.t. it becomes a chow if we add an appropriate tile with the same colour. A \emph{premeld} {\rm(}\emph{pmeld}{\rm)} is a pchow or a pair.  We say a tile $c$ \emph{completes} a pmeld $(ab)$ if $(abc)$ {\rm(}after reordering{\rm)} is a meld. Similarly, a pair is completed from a single tile $t$ if it is obtained by adding an identical tile to $t$.
\end{dfn}

For example, $(B3B4B5)$ is a chow, $(C1C1)$ is a pair, $(B7B7B7)$ is a pong, and $(B1B3)$ and $(C2C3)$ are two pchows.


In a Mahjong game, each player begins with a set of 13 tiles (called her \emph{hand}). Players may change a tile by either drawing a new tile or robbing a tile, say $t$, discarded by another player if the player has two identical tiles $t$ and wants to consolidate a pong $(ttt)$ or if $t$ completes her 13-tile into a winning hand. 
\begin{dfn}[hand]
\label{dfn:hand}
A $k$-tile is a sequence $S$ of $k$ tiles from $\M_0$, where no identical tiles can appear more than four times. A hand is a 13- or  14-tile.
\end{dfn}
To be precise, we write a tile $t$ as a pair $(c,n)$ s.t. $c \in \{0,1,2\}$ denotes the colour of the tile $t$ and $n$ denotes the number of $t$. Here $0, 1,2$ denote, respectively, Bamboo, Character, and Dot. For example, $(0,3)$ and $(1,5)$ denote $B3$ and $C5$ respectively. In addition, we assume that each $k$-tile $S$ is alphabetically ordered, e.g., $(0,3)$ is always before $(1,5)$ if  both appear in $S$.

Suppose $S$ is a sequence of $k$ tiles. Write $S[i]$ for the $(i+1)$-th value of $S$ for $0 \leq i \leq k-1$, i.e., we assume, as in Python and many other programming languages, that $S[0]$ represents the first entity of $S$.

\begin{dfn}[winning or complete hand, decomposition, eye \cite{li2019lets}]
\label{dfn:complete&decomposition}
A 14-tile over $\M_0$ is \emph{winning} or \emph{complete} if it can be decomposed into four melds and one pair.\footnote{For ease of presentation, we don't regard  a hand with seven pairs as complete.} Given a complete 14-tile $H$, a \emph{decomposition} $\pi$ of $H$ is a sequence of five subsequences of $H$ s.t. 
$\pi[4]$ is a pair and, for $0\leq i\leq 3$, each $\pi[i]$ is a meld. If this is the case, we call $\pi[4]$ the \emph{eye} of the decomposition.
\end{dfn}
The following 14-tile $H$ is complete as it has a decomposition $\pi$.
\begin{align}\label{mjsol1}
H &= (B1B2B2B3B3B4B7B7B7)(C1C1)(D4D5D6)\\ \nonumber
\pi &= (B1B2B3)(B2B3B4)(B7B7B7)(D4D5D6)(C1C1).
\end{align}

\begin{dfn} [deficiency \cite{li2019lets}]
\label{dfn:dfncy}
The \emph{deficiency number} \rm{(}or simply \emph{deficiency}\rm{)} of a 14-tile $H$ is defined recursively:
\begin{itemize}
    \item $H$ has deficiency 0 if it is complete;
    \item In general, for $\ell\geq 0$, $H$ has deficiency $\ell+1$ if it has no deficiency smaller than or equal to $\ell$ and there exists a tile $t$ in $H$ and another tile $t'$ s.t. $H[t/t']$  (the 14-tile obtained from $H$ by replacing $t$ with $t'$) has deficiency $\ell$.
\end{itemize}
If the deficiency of $H$ is $\ell$, we write $\dfncy(H) = \ell$.
\end{dfn}
The above definition can also been extended to 13-tiles: A 13-tile $H$ has deficiency 1 if there exists a tile $t$ s.t. $H$ becomes complete after adding $t$. In this case, $H$ is called a ready hand or  \emph{tenpai} in Japanese Mahjong.  In general, a 13-tile $H$ has deficiency $\ell+1$ for $\ell>1$ if it does not have deficiency smaller than or equal to $\ell$ and it has deficiency $\ell$ after some proper tile change.

A 14-tile $H$ is often not complete. The deficiency number of $H$ measures how bad $H$ is by counting the number of necessary tile changes to make it complete. 

\begin{proposition}
\label{prop:dfncy<=6}
The deficiency of any 14-tile $H$ is not greater than 6.
\end{proposition}

For example, the following hand has deficiency 6.
\begin{align*}
H &=(B3B5B6B9)(C2C2C3C6C9)(D1D1D2D5D8)
\end{align*}

It is one basic requirement for a Mahjong AI  to tell how bad a hand is when playing Mahjong. To tell if her hand $H$ is winning, the agent only needs to check if $\dfncy(H)=0$. In general, after drawing a tile from the Wall, the agent needs to select a tile to discard if it does not complete her hand. This can be done by comparing, for each tile $t$ in $H$, if she has the maximum chance of decreasing the deficiency after discarding $t$ and drawing a tile from the Wall. Analogously, suppose another player discards a tile $t$ and the agent has a pair $(tt)$ in her hand. She needs to decide if she performs a pong of $t$. She may adopt a strategy in which she performs the pong only when the action helps to reduce (or does not increase) the hand deficiency.  

\section{The Quadtree Algorithm for Deficiency Computing} \label{sec:quadtree}

In this section we describe the quadtree algorithm for deficiency computing introduced in \cite{li2019lets}, which is based on calculating and comparing the costs of all pre-decompositions of a 14-tile $H$. 

In this paper, we regard any sequence of tiles as a multiset, allowing for multiple instances for each of its elements. The usual set operations $\cup,\cap,\setminus$ are also defined for multisets. For clarity, we also write a multiset, e.g., $\{B5,B5,B6,C8,C9\}$, as ordered tuples like $(B5B5B6)(C8C9)$.

\subsection{Pre-decomposition}

Given a 14-tile $H$, we can determine if $\dfncy(H)\leq k$ by recursively checking if it has a neighbour $H[t/t']$ with deficiency $<\!k$, where $H[t/t']$ denotes the hand obtained by replacing a tile $t$ in $H$ by another tile $t'$.
This is extremely inefficient. The more practical quadtree method is  based on the notion of pre-decomposition.
\begin{dfn}[pre-decomposition and its remainder \cite{li2019lets}]
A \emph{pre-decomposition} {\rm(}\emph{{\pd}}{\rm)} is a sequence $\pi$ of five sequences, $\pi[0],...,\pi[4]$, s.t. 
\begin{itemize}
    \item $\pi[4]$ is a pair, a single tile, or empty;
    \item for $0\leq i\leq 3$, each $\pi[i]$ is a meld, a pmeld, a single tile, or empty.
\end{itemize}
We call each $\pi[i]$ for $0\leq i\leq 3$ a meld holder and call $\pi[4]$ the eye holder.
We say $\pi$ is a complete {\pd} if it is a decomposition. If $\bigcup_{i=0}^4 \pi[i]$ is contained in a 14-tile $H$, we say $\pi$ is a $\pd$ of $H$ and call the set of tiles in $H$ that are not in $\bigcup_{i=0}^4 \pi[i]$ the \emph{remainder} of $H$ under $\pi$.  
\end{dfn}
A decomposition of a complete hand $H$ is also a {\pd} of $H$.
\begin{dfn}[completion and cost \cite{li2019lets}]
\label{dfn: dcmp_cost}
Suppose $\pi$ and $\pi'$ are two {\pd}s. We say $\pi'$ is \emph{finer} than $\pi$ if, for each $0\leq i\leq 4$, $\pi[i]$ is identical to or a subsequence of $\pi'[i]$. A {\pd} $\pi$ is \emph{completable} if there exists a decomposition $\pi^*$ that is finer than $\pi$. If this is the case, we call $\pi^*$ a \emph{completion} of $\pi$.  

The \emph{cost} of a completable {\pd} $\pi$, written $\cost{\pi}$,  is the number of missing tiles required to complete $\pi[4]$ into a pair and complete each $\pi[i]$ into a meld for $0\leq i\leq 3$. If $\pi$ is incompletable, we say it has infinite cost.
\end{dfn}

A 14-tile $H$ may have many different {\pd}s.

\begin{myexp}
\label{ex:3pd}
Consider the 14-tile 
\begin{align*}
H &=(B1B5B6B8B8B8B9)(\cdot)(D1D2D4D5D5D6D7)
\end{align*}
It is easy to check that the following are all its {\pd}s 
\begin{align*}
\pi_0 &= (B5B6)(B8B8)(D4D5D6)(D5D7)(B8)\\
\pi_1 &= (B5B6)(B8B8B8)(D4D5D6)(D5D7)(B1) \\
\pi_2 &= (B5B6)(B8B8B8)(D4D5)(D5D6D7)(B1).
\end{align*}
The meld holder $(B8B8)$ and the eye holder $(B8)$ in $\pi_0$ cannot be both completed, as there are only four $B8$ in Mahjong. The {\pd}s $\pi_1$ and $\pi_2$ have, respectively, the following completions
\begin{align*}
\pi_1^* &= (\underline{B4}B5B6)(B8B8B8)(D4D5D6)(D5\underline{D6}D7)(B1\underline{B1})\\
\pi_2^* &=  (\underline{B4}B5B6)(B8B8B8)(D4D5\underline{D6})(D5D6D7)(B1\underline{B1}).
\end{align*}
We have $\cost{\pi_1}=\cost{\pi_2}=3$. It is easy to check that  $\dfncy(T)=3$. Therefore, both $\pi_1$ and $\pi_2$ are  {\pd}s with the minimal cost. 
\end{myexp}
\begin{theorem}[\cite{li2019lets}] \label{thm:m-cost}
The deficiency of a 14-tile $H$ is the minimum cost of its {\pd}s. We say a {\pd} $\pi$ of a 14-tile $H$ is \emph{minimal} if $\cost{\pi} = \dfncy(T)$.
\end{theorem}

 We next present the deficiency calculation method introduced in \cite{li2019lets}.

 \subsection{The quadtree method}

The quadtree method determines the deficiency of a 14-tile $H$ by constructing  and evaluating its possible {\pd}s in an exhaustive way. In the quadtree, each node is represented by a word $\alpha$ in the alphabet $\Sigma=\{1,2,3,4\}$, where each symbol represents a possible action. Each node $\alpha$ is attached with a {\pd} $\pi_\alpha$ of $H$ and a subsequence $S_\alpha$ of $H$, which denotes the set of tiles remaining to be processed. As will become clear,  two words $\alpha,\beta$ are identical iff $(\pi_\alpha,S_\alpha) = (\pi_\beta,S_\beta)$. In what follows, we do not distinguish between a node $\alpha$ and its associated $(\pi_\alpha,S_\alpha)$. In the construction process, when $S_\alpha$ becomes empty,
we terminate the branch with an exact cost estimation. Note that it is possible that $\pi_\alpha$ may still have empty placeholders and can be refined by recycling tiles from the remainder of $\pi_\alpha$. For each empty placeholder, we decrease the cost by 1. This estimation is exact except a few exceptions that are treated specially (see \cite{li2019lets}).   


An illustration of the search method is given in Fig.~\ref{fig:quadtree}.

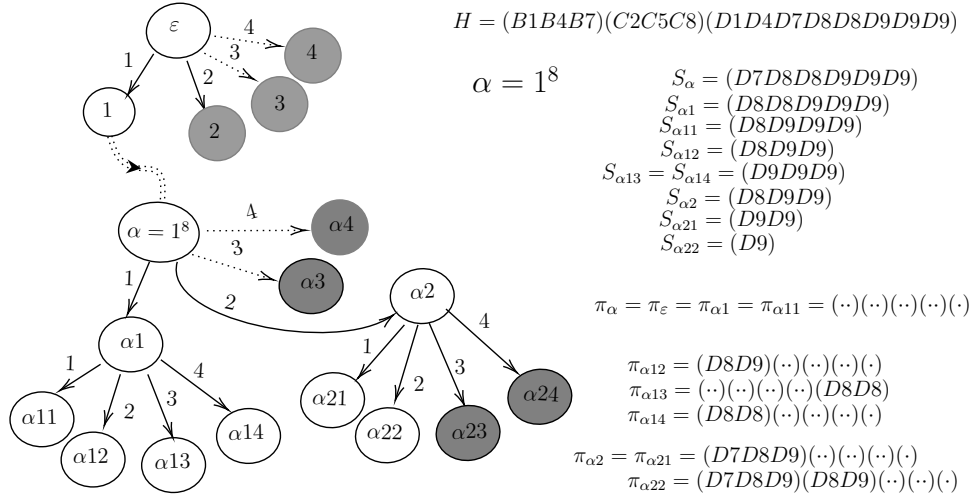
\begin{figure}[htb]
\begin{center}
\scalebox{0.7}{
\input{pic/quadtree}
}
\end{center}
\caption{The quadtree (partial) of a 14-tile $H$, where black nodes are closed nodes that will not be put into $Q$.}\label{fig:quadtree}
\end{figure}

We use a queue $Q$ to keep all nodes to be explored and a value $\val$ to denote the current minimum cost. The root node is denoted by the empty word $\varepsilon$, which has  $S_\varepsilon=H$ and $\pi_\varepsilon$ defined as $\pi_\varepsilon[i]=\varnothing$ for $0\leq i\leq 4$. Initially, we set $\val=6$, the largest possible deficiency number and put the root node $\varepsilon$ in $Q$. 
 
Suppose a node $\alpha=(S_\alpha,\pi_\alpha)$ is popped out from the queue $Q$. 
We \emph{expand} $\alpha$ as follows. Let $t=S_\alpha[0]$ be the first tile in the subsequence $S_\alpha$. We add  up to four child nodes $\alpha1$ to $\alpha4$ under the node $\alpha$. For each $\ell\in\{1,2,3,4\}$, $S_{\alpha\ell}$ and $\pi_{\alpha\ell}$ are obtained by,  respectively, \emph{reducing} $S_\alpha$ and \emph{refining} $\pi_\alpha$ as follows. 
 
 \paragraph{Notation}\quad Given a tile $t=(c,n)$, we write $t^+$ and $t^{++}$ for $(c,n+1)$ and $(c,n+2)$, respectively, as long as they are legal Mahjong tiles.
 
 \begin{itemize}
     \item [1.] (pass) Define $S_{\alpha1}=S_\alpha\setminus (t)$ and $\pi_{\alpha1}=\pi_\alpha$.
     
     \item [2.] (make chow) If $t^+$ or $t^{++}$ is in $S_\alpha$,
     define $S_{\alpha2}=S_\alpha \setminus (tt^+t^{++})$.     Suppose $i$  is the first index in $\{0,1,2,3\}$ s.t.  $\pi_\alpha[i]=\varnothing$. Define $\pi_{\alpha2}[i]=(tt^+t^{++})\cap S_\alpha$  and $\pi_{\alpha2}[j]= \pi_\alpha[j]$ for $j\not=i$.
    \item [3.] (make eye) If $(tt)\subseteq S_\alpha$ and $\pi_\alpha[4]$ is empty, define $S_{\alpha3} = S_\alpha \setminus (tt)$ and $\pi_{\alpha3}[4]=(tt)$ and $\pi_{\alpha3}[j]=\pi_\alpha[j]$ for $0\leq j\leq 3$.
    
    \item [4.] (make pong) If $(tt)\subseteq S_\alpha$,  define  $S_{\alpha4} = S_\alpha \setminus (ttt)$.
    Suppose $i$ is the first index in $\{0,1,2,3\}$ s.t. $\pi_\alpha[i]=\varnothing$. Define $\pi_{\alpha4}[i]=(ttt)\cap S_\alpha$  and $\pi_{\alpha4}[j]= \pi_\alpha[j]$ for $j\not=i$.
 \end{itemize}
 If a certain condition is not satisfied, e.g., when neither $t^+$ nor $t^{++}$ is in $S_\alpha$, this node is not introduced. For each child node $\alpha\ell$, if $\pi_{\alpha\ell}[i]\not=\varnothing$ for all  $0\leq i\leq 4$, we compare the value $\cost{\pi_{\alpha\ell}}$ with the current best value $\val$, update $\val$ as $\cost{\pi_{\alpha\ell}}$ if the latter is smaller, and terminate this branch. 

Suppose $\pi_{\alpha\ell}[i]$ is empty for some $i$. If there are tiles left to process, i.e., $S_{\alpha\ell}\not=\varnothing$, we put $\alpha\ell$ in $Q$ and expand it later; otherwise, we derive the cost of a refinement of $\pi_{\alpha\ell}$ that has no empty placeholder (see \cite{li2019lets} for details) and compare that cost with $\val$, update $\val$ accordingly, and terminate this branch. If $Q$ is nonempty and $\val>0$, we pop out the next node $\beta$ from $Q$ and expand $\beta$ as above. The procedure terminates when $\val=0$ or $Q=\varnothing$.

\begin{myexp}
\label{ex:T9}
Consider the 14-tile shown in Fig.~\ref{fig:quadtree}:
\begin{align*}
H &=(B1B4B7)(C2C5C8)(D1D4D7D8D8D9D9D9).
\end{align*}
Let $\mathcal{T}$ be the quadtree of the 14-tile $H$. For $1\leq k\leq 8$, $\alpha_k= {1^k} \equiv \underbrace{1\cdots 1}_k$ is the only depth $k$ node of $\mathcal{T}$ that are not closed. In particular, $\pi_{\alpha_8}=\pi_\varepsilon$ and $S_{\alpha_8}=(D7D8D8D9D9D9)$. As there are 6 tiles to be processed, we can make at most two melds or one meld and one eye, which left three places of the {\pd} empty. This implies that, for any node $\beta$ that is a descendant of $\alpha_9=1^9$, we have  ${\cost{\pi_\beta}}\geq 6\geq \val$. Thus, it is not necessary to expand $\alpha_9$. 
\end{myexp}



Consider the 14-tile and {\pd}s in Example~\ref{ex:3pd} again. None of these {\pd}s is associated with a node in the quadtree, but $\pi_0$ is a direct refinement of $\pi_{\alpha}$ with $\alpha=121411122$ by putting $B8$ in the eye holder. Analogously, $\pi_1$ is a direct refinement of $\pi_{\beta}$ with $\beta=12411122$ by putting $B1$ in the eye holder. However, even if we removing $B1$ from the eye holder, $\pi_2$ is still not associated with any node in the quadtree. This is because, when $S_\gamma=(D4D5D5D6D7)$ and we examine the tile $D4$, if we select action 2 and put $D4$ and $D5$ in a meld holder, then according to the quadtree construction we should also put $D6$ together with $D4$ and $D5$. This does not affect the correctness of the method as long as the current knowledge of available tiles is not concerned (cf. Example~\ref{ex:TT2} below).

In the worst case, the quadtree has a maximum depth of 14, as $S_{\alpha_k}$ with $\alpha_k = 1^k$ is a node in the quadtree with length $k$ ($0\leq k\leq 14$). Often, we may stop earlier if we know for certain that no refinement of $\pi_\alpha$ can have a cost smaller than $\val$.

Since this exhaustive method is inefficient, in practice, we have to adopt an approximate method by, e.g., considering only nodes with depth smaller than or equal to 8. While this is sufficient in most cases, it is still not very efficient and,  in some cases, is inexact and may lead to bad decisions. 

\section{Knowledge-aware Deficiency}\label{sec:kb-deficiency}
Another problem with the quadtree method is that it does not take into consideration the agent's knowledge about the game state. In this section, we propose a modified definition of the deficiency number, which respects the agent's knowledge about the available tiles in the current state of the game.
\subsection{Knowledge base and knowledge-aware deficiency}
\label{sec:kb&qd}
The agent's knowledge base contains all her information of available tiles.
\begin{dfn}[knowledge base \cite{li2019lets}] \label{dfn:KB}
The knowledge base  of the agent is a 27-tuple $KB$. For each tile $t=(c,n)$ ($0\leq c\leq 2$, $1\leq n \leq 9$),  $KB(t) = KB(c,n)\define KB[9c+n-1]$  denotes the multiplicity of  $t$ the agent \emph{believes} to be available. 
\end{dfn}
In what follows, we say $KB$ has a tile $t$, or say $t$ is in $KB$, if $KB(t)>0$. Similarly, we say $KB$ has a pair (pong) if there is a tile $t$ s.t. $KB(t)\geq 2$ ($KB(t)\geq 3$) and say $KB$ has a chow if there is a chow $(t_1,t_2,t_3)$ s.t. $\min(KB(t_1),KB(t_2),KB(t_3))\geq 1$.

Initially, we have $KB(t) = 4$ for each tile $t$. When all players have their hands, the agent also has her hand $H$ and updates her $KB$ accordingly as
\begin{align}\label{eq:omega}
KB(t) = 4- \mbox{the number of  $t$ in $H$}. 
\end{align}
Then she continues to modify $KB$ according to the process of the game. For example, if one player discards a tile $t$ and a pong of $t$ is made by some other player, then the agent updates her $KB$ decreasing by 3 her $KB(t)$ and leaves the other items unchanged. The agent may also update her $KB$ by reasoning. For example, suppose player $A$ declared win after robbing a tile $C2$. If the agent's $KB$ shows that $C2$ and $C4$ are not available, she can infer that player $A$ must have $C1$ and $C3$ in his hand (to make a meld $(C1C2C3)$). Accordingly, the agent could update her $KB$ decreasing by one her $KB(C1)$ and $KB(C3)$.

In what follows, whenever a $k$-tile $H$ and a knowledge base $KB$ appear together, we always assume that they are \emph{compatible} in the sense that the number of any tile in $H$ and $KB$ is not greater than 4. 



If every tile that can complete a hand $H$ has been discarded, ponged, or robbed,  $H$ becomes incompletable in the current $KB$. This suggests that   Definition~\ref{dfn:dfncy} should take $KB$ into consideration. 

For convenience, for a $k$-tile $H$, a knowledge base $KB$, and any tile $t$, we write $H\oplus t$ as the $(k+1)$-tile obtained by adding $t$ to $H$ (after reordering), and  $KB\ominus t$ as the updated knowledge base which differs from $KB$ only in that it has one less $t$ than $KB$ does.

\begin{dfn} [knowledge-aware deficiency]
\label{dfn:kb-dfncy}
The \emph{deficiency number} \rm{(}or  \emph{deficiency}\rm{)} of a 14-tile $H$ w.r.t. a knowledge base $KB$ is defined recursively:
\begin{itemize}
    \item $H$ has deficiency 0 if it is complete;
    \item In general, for $\ell\geq 0$, $H$ has deficiency $\ell+1$ if it has no deficiency smaller than or equal to $\ell$ and there exists a tile $t$ in $H$ and another tile $t'$ that is available in $KB$ (i.e., $KB(t')>0$) s.t. $H[t/t']$   has deficiency $\ell$ w.r.t. the updated knowledge base $KB \ominus t'$.
\end{itemize}
If the deficiency of $H$ is $\ell$, we write $\dfncy(H, KB) = \ell$. We say $H$ is incompletable if it has no finite deficiency.

Suppose $H$ is a 13-tile. The deficiency of $H$ w.r.t. a knowledge base $KB$ is defined as the minimum of 
    $\dfncy(H\oplus t, KB \ominus t)+1$ over all tiles $t$ with $KB(t)>0$. 
\end{dfn}

It is easy to see that, when ignoring the knowledge base $KB$, i.e., interpreting $KB$ as Eq.~\eqref{eq:omega}, then $\dfncy(T)=\dfncy(T,KB)$.


\begin{myexp}
\label{ex:TT2}
Consider the 14-tile $H$ and its knowledge base $KB$
\begin{align*}
H &=(B1B5B6B8B8B8B9)(\cdot)(D1D2D4D5D5D6D7),\\
KB &= (343423023)(434434443)(334220344).
\end{align*}
For clarity, we group the 27-tuple $KB$ into three 9-tuples, representing the local knowledge bases of colour $B,C,D$. In the quadtree search process, the {\pd} at node $\alpha=(12311222)$ is $\pi= (B5B6)(D1D2)(D4D5D6)(D5D7)(B8B8)$. If we don't consider the knowledge base, $\pi$ has cost 3. Let $\pi^*=(B5B6)(D1D2)$ $(D4D5)(D5D6D7)(B8B8)$. Then $\pi^*$ is another {\pd} with cost 3. However, $\pi^*$ is not the {\pd} of any node in the quadtree! This is because, by construction, if $(D4D5D6)$ is a subsequence of $S_\beta$ and $D4=S_\beta[0]$, then we shall enrich $\pi_{\beta}$ with the chow $(D4D5D6)$ instead of a pchow $(D4D5)$.   

Taking $KB$ into consideration, $\pi$ is incompletable under $KB$ as it has a pmeld $(D5D7)$ but $D6$ is not available in $KB$. In comparison, $\pi^*$ is still completable under $KB$ and has cost 3 in the sense of Definition~\ref{dfn: dcmp_cost}.
\end{myexp}
This example shows that the quadtree method is incomplete when the knowledge base is taken into consideration, as some {\pd}s with the minimal cost are not reachable. In order to get a complete method, we could expand the quadtree into a 6-ary tree, where each node have up to six child nodes s.t. pmelds like  $(D4D5)$ and $(D4D6)$ could appear in the 6-ary tree of the above example. However, this will make the algorithm even  slower (cf. Sec.~\ref{sec:experiment} for evaluation). Therefore, it is necessary to devise a new deficiency calculation algorithm which is efficient and respects the knowledge base. To this end, we introduce the following adapted notion of quasi-decompositions and reduce the deficiency of a hand w.r.t. to a knowledge base to the minimal cost over all its quasi-decompositions.

\subsection{Quasi-decompositions}

We need the following notion of completable pmeld. 
\begin{dfn} [completable pmeld] \label{dfn:completable_pmeld}
A pchow $(t,t')$ is completable under $KB$ if $KB$ has a tile $t''$ s.t. $(t,t',t'')$ (after reordering) is a chow. Similarly, a pair $(t,t)$ is completable under $KB$ if $KB$ has an identical copy of $t$ (i.e., $KB(t)>0$). A pmeld is completable if it is a completable pchow or a completable pair.  
\end{dfn}

A quasi-decomposition is a variant of pre-decomposition that respects the knowledge base but does not fix the last place as the eye holder.
\begin{dfn} [quasi-decomposition] \label{dfn:qdcmp}
Let $H$ be a $k$-tile and $KB$ a knowledge base. A quasi-decomposition ({\qd}) $\pi=(\pi[0],\pi[1],\cdots$, $\pi[k])$ of $H$ w.r.t.  $KB$ is a set of (possibly identical) subsequences of $H$ s.t. 
\begin{itemize}
    \item  $k\leq 4$ and each $\pi[i]$ is a meld, a pair, or a pchow. 
    \item  If $k=4$, $\pi$ contains at least one pair.
    \item Except at most one pair, all pmelds in $\pi$ are completable under $KB$.
    \item $\bigcup_{i=0}^k \pi[i] $ is contained in $H$. 
\end{itemize}
 The \emph{remainder} of $\pi$ in $H$ is the sequence of tiles in $H$ that are not in $\bigcup_{i=0}^k \pi[i] $.
\end{dfn}

\begin{myexp}
\label{ex:running_example}
Let 
\begin{align}
    H &= (\cdot)(C1C4C6C7C8C9)(D1D2D3D6D6D7D8)\\
    KB & = (001100121)(010000030)(032242321)
\end{align}   
As the knowledge base contains only one $C2$ and three $C8$, pchows $(C4C6)$, $(C6C8)$, $(C7C8)$, and $(C8C9)$ are incompletable. 
Then 
\begin{align*}
\pi_0 &= (\cdot),\\
\pi_1 &=((C6C7),(D1D2D3), (D6D6),(D7D8)),\\
\pi_2 &=((C7C9),(D1D2D3), (D6D6),(D7D8) ),\\ 
\pi_3 &=((C6C7C8),(D1D2D3),(D6D6), (D7D8)),\\ 
\pi_4 &=((C7C8C9),(D1D2D3),(D6D6), (D7D8))
\end{align*}
are examples of valid {\qd}s of $H$. In addition, although $(C6C7)$ and $(C7C9)$ are completable, they can only be completed by $C8$, which is also in the remainder of the corresponding ${\qd}$s, viz. $\pi_1$ and $\pi_2$. Intuitively, their completion costs are larger than that of $\pi_3$ and $\pi_4$ and thus it is not necessary to evaluate them. In practice, we need only evaluate $\pi_3$ and $\pi_4$. 
\end{myexp}

Suppose $\pi$ is a  {\qd} of $H$ under $KB$ and $R$ the remainder of $\pi$ in $H$. The cost of $\pi$ under $KB$, written $\cost{\pi,KB}$, is the minimal number of tile changes to complete $\pi$ into a decomposition by applying the following procedure until we are certain that $\pi$ is incompletable or have four melds and one eye.
\begin{itemize}
    
    \item [(a)] If $\pi$ has an incompletable pair, select it as the eye; otherwise, select a pair in $\pi$ to act as the eye or create the eye from scratch. 
    
    \item [(b)] Complete every pmeld in $\pi$ that does not act as the eye.
    
    \item [(c)] If by far $\pi$ has three or less melds, create a new meld from scratch until $\pi$ has four melds.
\end{itemize}
More precisely, suppose $(t,t')$ is a pmeld in $\pi$ which we want to complete. If $(t,t')$ is a pchow, we borrow one tile from $KB$ and complete it into a chow, increasing the cost by 1. Suppose $(t,t)$ is a pair that is not selected as the eye in Step (a). Since it is completable, we borrow one tile $t$ from $KB$ (and increase the cost by 1) to complete $(t,t)$ into a pong.

To create the eye from scratch, we first examine if there is a tile $t$ in $R$ which has an identical tile in $KB$.\footnote{It is possible that $R$ contains a pair $(t,t)$, we don't regard the cost to create the eye as 0.  This is because the deficiency is determined by considering all {\qd}s and $(t,t)$ appears in some other {\qd} $\pi'$ that is finer than $\pi$ and has a smaller cost.}   If so, we use $(t,t)$ as the eye and this incurs cost 1. Suppose otherwise. We check if $KB$ has a pair; if so, we use it as the eye with the cost increased by 2. In case that $KB$ contains no pairs, we cannot create a new pair and need to go back to Step (a) and check if we can use a pair in $\pi$ as the eye. If that is still impossible, then $\pi$ is incompletable.

Analogously, to create a new meld from scratch, we increase the cost by 2 if it is possible to make a meld by using one tile from $R$ and borrowing two tiles from $KB$; otherwise, we increase the cost by 3 if $KB$ contains a meld. If neither is true, then $\pi$ is incompletable.

Suppose the procedure terminates with a decomposition $\pi^*$ and the total cost is $c$. We call $\pi^*$ a completion of $\pi$ and $c$ the cost of $\pi$, written as $\cost{\pi,KB}$. 

Similar to Theorem~\ref{thm:m-cost}, knowledge-aware deficiency has the following characterisation. 
\begin{theorem} \label{thm:m-cost-kb}
Let $H$ be a hand and $KB$ a knowledge base. The deficiency of $H$ w.r.t. $KB$ is the minimal of $\cost{\pi,KB}$ for all  {\qd}s $\pi$ of $H$ under $KB$. 
\end{theorem}

The next subsection shows that the cost of a {\qd} can be determined by its attributes. 

\subsection{Attributes of a \qd} \label{sec:attr}
Let $H$ be a $k$-tile and $KB$ a knowledge base. Suppose  $\pi$ is a {\qd} of $H$ and $R$ the remainder of $\pi$ in $H$. We introduce the following attributes for $\pi$, where $m,n,p,e$ are non-negative integers and $re,rm,em$ are Boolean values.
\begin{itemize}
    \item [$m$] The number of melds in $\pi$. 
    \item [$n$] The number of pmelds in $\pi$, including pchows and pairs, where each pchow is completable w.r.t. $KB$. 
    \item [$p$] The number of pairs in $\pi$. 
    \item [$e$] The number of pairs in $\pi$ which are incompletable. 
    \item [$re$] If the remainder has a tile $t$ which is also in $KB$, then $re=1$; otherwise, it's 0. If $re=1$, we can make the eye based on one tile recycled from $R$. 
    \item [$rm$] If the remainder has a tile $t$ which can evolve into a meld (i.e., a chow or a pong) given the current $KB$, then $rm=1$; otherwise, it's 0. If $rm=1$, we can make a meld based on one tile recycled from $R$. 
    \item [$em$] Suppose $e=0$ and $re=rm=1$ and we need to make a meld and the eye both from scratch. If we cannot make a meld from a tile in $R$ after making the eye starting from a tile in $R$,  or vice versa, then we say there is an \emph{eye-meld conflict} and set $em=1$. In all other cases, we set $em=0$.
\end{itemize}
In addition, we have the following attributes for $KB$.
\begin{itemize}
    \item [$ke$] If the $KB$ contains a pair, then $ke=1$; otherwise, it's 0. 
    \item [$km$] If the $KB$ contains a meld, then $km=1$; otherwise, it's 0.
\end{itemize}

\paragraph{About the notations} {The notation $xy$ for $x\in\{r,k\}$ and $y\in\{e,m\}$ reads as, if $y=e$ ($y=m$, resp.), we can make the {\bf e}ye (a {\bf m}eld, resp.) from $r$ (the {\bf r}emainder) or $k$ (the {\bf k}nowledge base).}

The following result is clear from the definition.
\begin{lemma}\label{lem:7attr}
Let $H$ be a $k$-tile, $KB$ a knowledge base, and $\pi$ a {\qd} of $H$ under $KB$. Suppose $m,n,p,e$ are attributes of $\pi$ defined above. Then  $e\leq 1$, $m+n\leq 5$, $3m+2n\leq k$, and $n\geq p\geq e\geq 0$. If $m+n=5$, then $p>0$.
\end{lemma}

\begin{dfn}[type of {\qd}] \label{dfn:type}
Let $H$ be a  $k$-tile and $KB$ a knowledge base. Suppose  $\pi$ is  a {\qd} of $H$ under $KB$. Let $m,n,p,e,re,rm,em$ be attributes of $\pi$ defined as above. We call $\sigma(\pi) \define (m,n,p,e,re,rm, em)$ the type of $\pi$.
\end{dfn}

We next show how the cost of $\pi$ under $KB$ can be determined by its type.

\begin{lemma}\label{lem:incompletable}
Let $\pi$ be a  {\qd} of $H$ under $KB$. Suppose $\sigma(\pi)=(m,n,p$, $e,re,rm,em)$ is the type of $\pi$ and $ke,km$ the two attributes of $KB$. Then $\pi$ is incompletable if 
\begin{itemize}
\item [\rm(a)] $p=re=ke=0$; or 
\item [\rm(b)] $m+n\leq 4$, $rm=km=0$, and $re=ke=0$; or   
\item [\rm(c)] $m+n-e\leq 3$ and $rm=km=0$; or  
\item [\rm(d)] $m+n\leq 3$, $p=ke=km=0$ and $re=rm=em=1$.
\end{itemize}
\end{lemma}



The above conditions are sufficient but not necessary. However, we can easily see from the following two lemmas that they are necessary when $m+n \geq 4$, or when $m+n=3$ and $p=0$.

We next examine the cases when \red{$m+n \geq 4$} and $m+n-e=3$.
\begin{lemma}\label{lem:m+n>=4}
Let $\pi$ be a {\qd} of $H$ under $KB$ and  $(m,n,p,e,re$, $rm,em)$ its type. Suppose $m+ n \geq 4$. 
\begin{itemize}
    \item [\rm(a)] If $m+n>4$, then the cost of $\pi$ is $4-m$.
    \item [\rm(b)] Suppose $m+n=4$. If $e=0$ and $re=1$, or $p>0$ and $rm=1$,  then the cost of $\pi$ is $4-m+1$.
    \item [\rm(c)] Suppose $m+n=4$,  
    $re=0$ if $e=0$, and $rm=0$ if $p>0$.\footnote{Here a statement `$B$ if $A$' is equivalent to saying `$A\Rightarrow B$' or  `(not $A$) or $B$'. } 
    If $e=0$ and $ke=1$, or $p>0$ and $km=1$, then the cost of $\pi$ is $4-m+2$. 
    \item [\rm(d)] If none of the above, then $\pi$ is incompletable.   
\end{itemize}
\end{lemma}




Note in this case the eye-meld conflict does not matter as we do not need to make the eye and a new meld simultaneously. The next lemma considers the cases when the number of melds and completable pmelds in $\pi$ is three. 
\begin{lemma}\label{lem:cost3}
Let $\pi$ be a  {\qd} of $H$ under $KB$ and  $(m,n,p,e,re$, $rm,em)$ its type. Suppose $m+ n-e = 3$ and $\max(p, re,ke)>0$ and $\max(rm,km)=1$. Let $mcost=2\times rm+3\times(1-rm)$ and $ecost=re+2\times(1-re)$. 
\begin{itemize}\label{lem:cost2}
    \item [\rm(a)] If $e=1$, then the cost of $\pi$ is $n-1 + mcost$.
    
    \item [\rm(b)]  If $p=em=0$, then the cost of $\pi$ is $n + mcost + ecost$.

    \item [\rm(c)]  If $p=0$, $em=1$, and $\max(ke,km)=1$, then the cost of $\pi$ is $n + 4\geq 4$.
    
    \item [\rm(d)]  If $p>e=0$, then the cost of $\pi$ is at least $\min(f_1,f_2)\geq 4$, where $f_1 \define n + mcost + ecost$ and $f_2 \define n-1 + 2\times mcost$.
\end{itemize}
\end{lemma}

When $m+n-e\leq 2$, we can analyse as above. In order to get the exact cost, we need to introduce new attributes and consider many more subcases. This is in general not necessary as the cost of a $\qd$ in this case is at least 4, which is quite bad and usually of little use as it indicates that either (i) the {\qd} is not a good one and may be replaced with a better one, or (ii) the hand is pretty bad, and, if this is in the late stage of the game, we  may have no chance to complete the hand. 
\begin{lemma}
Let $\pi$ be a {\qd} of $H$ under $KB$ and  $(m,n,p,e,re,rm,em)$ the type of $\pi$. If $m=0$ or $m+n-e\leq 2$, then the cost of $\pi$ is at least 4. 
\end{lemma}
When $m+ n-e \leq 2$, similar to Lemma~\ref{lem:cost2}, we have the following result.
\begin{lemma}\label{lem:cost>4}
Let $\pi$ be a {\qd} of $H$ under $KB$ and  $(m,n,p,e,re,rm,em)$ the type of $\pi$. Suppose $m+ n-e \leq 2$. Let $mcost=2\times rm+3\times(1-rm)$ and $ecost=re+2\times(1-re)$. 
\begin{itemize}
    \item [\rm(a)] Suppose $e=1$. The cost of $\pi$ is at least $n-1 + mcost\times (4-m-n+1) \geq 4$.
    
    \item [\rm(b)]  Suppose $p=0$. The cost of $\pi$ is at least $n + mcost\times (4-m-n) + ecost + em \geq 5$.
    
    \item [\rm(c)]  If $p>e=0$, then the cost of $\pi$ is at least $\min(f_1,f_2)\geq  6$, where $f_1 \define n + mcost\times (4-n-m) + ecost$ and $f_2 \define n-1 +  mcost\times (4-m-n+1)$.
\end{itemize}
\end{lemma}

\begin{myexpcont}\label{ex:running_example}
Recall
\begin{align*}
    H &= (\cdot)(C1C4C6C7C8C9)(D1D2D3D6D6D7D8)\\
    KB & = (001100121)(010000030)(032242321).
\end{align*}    
As $KB$ contains pairs and melds, we have $ke=km=1$.
Consider the {\qd} $\pi\define(C6C7C8)(D1D2D3)(D6D7D8)$ with remainder $R\define(C1C4C9D6)$. It is easy to check that $\pi$ has type $(3,0,0,0,1,1,1)$. Since $ke=km=1$, by Lemma~\ref{lem:cost3} (c), the cost of $\pi$ is 4. After examining all {\qd}s of $H$, we know this is the minimal cost and thus $H$ has deficiency 4 under $KB$, i.e., $\dfncy(H,KB)=4$. 
\end{myexpcont}

We integrate the above results in Alg.~\ref{alg:decide} to determine the cost of a {\qd}.

\begin{algorithm}
\caption{$Decide(m,n,p,e,re,rm,em,ke,km)$}
\begin{algorithmic}[1]
\Require{A {\qd} type $\vec{x}=(m,n,p,e,re,rm,em)$ and $ke,km$.}
\Ensure{The cost of $\vec{x}$.}

\IIf{$p=re=ke=0$}{\ \Return 100} \EndIIf 
\IIf{$m+n\leq 4$ and $re=ke=rm=km=0$}{\ \Return 100} \EndIIf
\IIf{$m+n-e\leq 3$ and $rm=km=0$}{\ \Return 100} \EndIIf 
\IIf{$m+n\leq 3$, $p=ke=km=0$, and $em=1$}{\ \Return 100} \EndIIf

\If{$m+n\geq 4$}
    \If{$m+n>4$}
        \State \Return $4-m$ 
    \Else
        \If{($e=0$ and $re=1$) or ($p>0$ and $rm=1$)}
            \State \Return $4-m+1$
        \Else 
            \State \Return $4-m+2$
        \EndIf
    \EndIf 
\EndIf 
\If{$m+n-e\leq 3$}
    \State  $mcost \leftarrow  2\times rm + 3\times (1-rm)$ 
    \State $ecost \leftarrow re + 2\times (1-re)$
    \If{$e=1$}
        \State \Return $(n-1) + mcost\times (4-m-n+1)$ 
    \Else
        \If{$p=0$} \Comment{$p=e=0$}
                \State \Return $n + mcost\times (4-m-n) + ecost + em$
        \Else \Comment{$p>e=0$}
            \State  $f1 \leftarrow n + mcost\times (4-m-n) + ecost$ 
            \State $f2 \leftarrow n -1 + mcost\times (4-m-n+1)$ 
            \State \Return $\min(f1,f2)$
        \EndIf
    \EndIf 
\EndIf 
\end{algorithmic}
\label{alg:decide}
\end{algorithm}

\section{The Block Deficiency Algorithm}
\label{sec:block_dfncy}

In the quadtree algorithm, the hand $H$ is considered as a whole and we generate, in principle,  all possible {\pd}s of $H$ and evaluate their costs one by one. This procedure can be significantly sped up by dividing the hand into blocks, generating the local quasi-decompositions and mapping them to small sets of local types and then merging these local types and evaluating their costs. To describe this block-based algorithm, we first introduce the notion of block.

\begin{dfn}[block]\label{dfn:block}
A block of a $k$-tile $H$ is a subsequence $b$ of $H$ s.t.
\begin{itemize}
    \item $b\not=\varnothing$  and all tiles in $b$ have the same colour;
    \item If $t$ is a tile in $b$, then any tile in $H$ that is connected to $t$ is also in $b$,
 \end{itemize}
 where two tiles $t=(c,n)$, $t'=(c',n')$ are connected if $c=c'$ and $|n-n'|\leq 2$.
\end{dfn}

Analogously to deficiency, blocks can also be knowledge-aware.
\begin{dfn}[knowledge-aware block]\label{dfn:block}
Given a $k$-tile $H$ and a knowledge base $KB$, a subsequence $b$ of $H$ is a $KB$-block if 
\begin{itemize}
    \item $b\not=\varnothing$  and all tiles in $b$ have the same colour;
    \item If $t$ is a tile in $b$, then any tile in $H$ that is $KB$-connected to $t$ is also in $b$,
 \end{itemize}
 where two tiles $t=(c,n)$ and $t'=(c',n')$ are $KB$-connected if (i) $t=t'$, or (ii)  either $H$ or $KB$ has a tile $t''$ s.t. $(tt't'')$ is a chow.
\end{dfn}
It is easy to see that there are blocks that are not $KB$-blocks and vice versa. In this paper, we are mainly concerned with $KB$-blocks. 

In this section, we use Example~\ref{ex:running_example} as our running example.

\begin{myexpcont}\label{ex:running_example}
Clearly, $H$ has the following $KB$-blocks
\begin{align*} 
b_1\define (C1), b_2\define (C4), b_3\define (C6C7C8C9), b_4\define (D1D2D3), b_5\define (D6D6D7D8).
\end{align*}
We note that $C4$ is not $KB$-connected to $C6$ as neither $H$ nor $KB$ contains a tile $C5$. All tiles in a $KB$-block, say $b_3=(C6C7C8C9)$, have the same colour. For simplicity, we write it as a list of integers, say $(6789)$, and write, say, $KB_1=(010000030)$ for the restriction of $KB$ to the Character suit that contains $b$.

The first block $b_1=(C1)$ has only the empty {\qd} $\pi_1=(\cdot)$, with remainder $R_1=(C1)$ and local knowledge base $KB_1=(010000030)$. As we cannot develop $C1$ into a pair or a meld, we have $re_1=rm_1=em_1=0$ and thus $\sigma(\pi_1)=(0,0,0,0,0,0,0)$. 
The second block $b_2=(C4)$ is similar. It has the empty {\qd} $\pi_2$ with type $(0,0,0,0,0,0,0)$. The third block $b_3=(C6C7C8C9)$ also has the  local knowledge base $KB_1$ in the Character suit.  For the {\qd} $\pi_3=((C6C7C8))$ of $b_3$, as its remainder is $(C9)$, we have $m_3=1$, $n_3=p_3=e_3=0$, and $re_3=rm_3=em_3=0$. 
Thus the type of $\pi_3$ is $(1,0,0,0,0,0,0)$.
Furthermore, consider the blocks $b_4=(D1D2D3)$ and $b_5=(D6D6D7D8)$ in the Dot suit, with the local knowledge base $KB_2=(032242321)$. Clearly, $\pi_4=((D1D2D3))$ and $\pi_5=((D6D7D8))$ are, respectively,  {\qd}s of $b_4$ and $b_5$ under $KB_2$. The type of $\pi_4$ is $(1,0,0,0,0,0,0)$. Let $(m_5,n_5,p_5,e_5,re_5,rm_5,em_5)$ be the type of $\pi_5$. As the remainder of $\pi_5$ is $(D6)$, we have $m_5=2$, $n_5=p_5=e_5=0$, $re_5=rm_5=1$. Since the remainder contains a single tile, we cannot make the eye and a meld from the remainder $(D6)$ simultaneously, i.e., we have an eye-meld conflict and $em_5=1$. Thus the type of $\pi_5$ is $(2,0,0,0,1,1,1)$.
\end{myexpcont}

\subsection{Description of the  new algorithm}

\begin{figure}[htb]
\begin{center}
\scalebox{0.75}{
\input{pic/block2}
}
\end{center}
\caption{Illustration of the block deficiency algorithm.}\label{fig:block}
\end{figure}
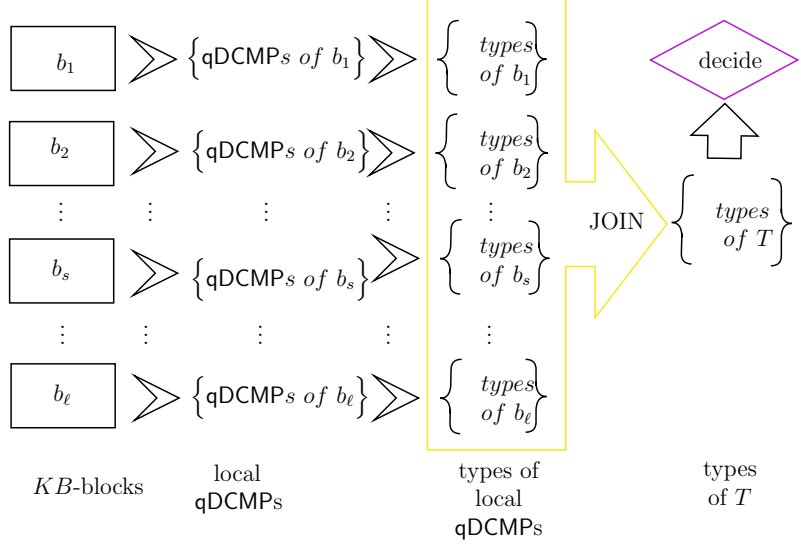

The idea is to first partition the hand $H$ into pairwise disconnected $KB$-blocks. Write $\mathcal{B}$ for the set of $KB$-blocks of $H$. For each $KB$-block $b$, we generate all its {\qd}s and call them \emph{local} {\qd}s at $b$. Instead of amalgamating these local {\qd}s directly (and form global {\qd}s), we calculate and join their \emph{types}. Here the type of a local {\qd} of $KB$-block $b$ is a 7-tuple $\vec{x}_b\define (m_b,n_b,p_b,e_b,re_b,rm_b,em_b)$ (cf. Definition~\ref{dfn:type}). Write $TypeSet_b$ for the set of types of all local {\qd}s of block $b$. We join these local types as follows. Let $\vec{x}_b$ be a type in $TypeSet_b$ for each $KB$-block $b$.  The global type  $\vec{x}\define (m,n,p,e,re,rm,em)$ is defined by setting 
\begin{align}\label{eq:em}
em=1\ \mbox{iff}\ (\exists b\in\mathcal{B}) \big(em_b=1 \wedge  (\forall b'\in\mathcal{B}) (b\not=b' \to re_{b'}=rm_{b'}=0) \big)
\end{align}
and setting the other attributes as follows:
\begin{equation}
\vec{x}[j] = \begin{cases}
            \sum_b \vec{x}_b[j] &\text{if  $0\leq j\leq 3$} \\
            \max_b \vec{x}_b[j] & \text{if $4\leq j\leq 5$}
        \end{cases}
\end{equation}
Note that in practice we join these local type sets one by one (see Alg.~\ref{alg:typecombine}), as the combined type set usually have size far less than the product of the sizes of two local type sets. For each global type $\vec{x}$, we determine its cost according to Alg.~\ref{alg:decide}, that is, the minimum number of tiles we should borrow from the knowledge base in order to complete the corresponding quasi-decomposition. Note that we ignore a global type $\vec{x}=(m,n,p,e,re,rm,em)$ if $e>1$ or $m+n>5$, as no legal {\qd} has type like that (cf. Lemma~\ref{lem:7attr}). Taking the minimum of the costs over all global types, we obtain the deficiency of the hand $H$.

Alg.~\ref{alg:block} presents the pseudocode of the block deficiency algorithm.

\begin{algorithm}
\caption{The block deficiency algorithm}
\begin{algorithmic}[1]
\Require{A 14-tile $H$ and a knowledge base $KB$.}
\Ensure{The deficiency of $H$.}
\State $\mathcal{B} \leftarrow \mbox{$KB$-blocks of }  H$

\Comment{Generate the local type set for each $KB$-block $b$}
\For {$b$ in $\mathcal{B}$}
    \State  $TypeSet_b  \leftarrow \varnothing$
    \State $DCMP_b \leftarrow$ all {\qd}s of $b$ \Comment{Analogous to the quadtree method} 
    \For {$\pi_b\in DCMP_b$}
        \State Add the type of $\pi_b$ to $TypeSet_b$
    \EndFor
\EndFor

\Comment{Combine the local type sets and get the global type set }
\State $TypeSet \leftarrow \varnothing$
\For{$b \in \mathcal{B}$}
    
    \If{$TypeSet = \varnothing$}
        \State $TypeSet  \leftarrow TypeSet_b$
    \Else
        \State $TypeSet  \leftarrow TypeSet \bowtie TypeSet_b$ \Comment{See Alg.~\ref{alg:typecombine}}
    \EndIf
\EndFor

\Comment{Evaluate each global type and obtain the deficiency}
\State $dfncy \leftarrow  100$ \Comment{100 is a large upperbound}
\State $ke, km \leftarrow 0, 0$
\IIf{$KB$ has a pair}{\ $ke\leftarrow 1$} \EndIIf 
\IIf{$KB$ has a meld}{\ $km\leftarrow 1$}\EndIIf
\For{$\vec{x} = (m,n,p,e,re,rm,em) \in TypeSet$}
    \State $dfncy' \leftarrow  Decide(m,n,p,e,re,rm,em, ke,km)$ \Comment{See Alg.~\ref{alg:decide}}
    \IIf{$dfncy'=0$}{\ \Return 0} \EndIIf
    \State $dfncy \leftarrow  \min(dfncy,dfncy')$
\EndFor
\State  \Return $dfncy$

\end{algorithmic}
\label{alg:block}
\end{algorithm}

\begin{algorithm}
\caption{Combine two type sets: $TypeSet \bowtie TypeSet'$}
\begin{algorithmic}[1]
\Require{Two type sets $TypeSet$ and $TypeSet'$.}
\Ensure{The combined type set $TypeSetX$.}

\State $TypeSetX \leftarrow \varnothing$
\For{$\vec{x} = (m,n,p,e,re,rm,em) \in TypeSet$}
    \For{$\vec{y} = (m',n',p',e',re',rm',em')  \in TypeSet'$}
    \State $m,n,p,e \leftarrow m+m', n+n', p+ p', e+e'$
    \If{$e>1$ or $m+n>5$ or ($m+n=5$ and $p=0$)} 
        \State \textbf{Continue}
    \Else
        \If{$(em=1,re'=rm'=0)$ or $(em'=1,re=rm=0)$}
            \State $em \leftarrow 1$
        \Else
            \State $em\leftarrow 0$ 
        \EndIf
        \State Add $(m,n,p,e,\max(re,re'),\max(rm, rm'),em)$ to $TypeSetX$
    \EndIf
\EndFor
\EndFor
\State Return $TypeSetX$

\end{algorithmic}
\label{alg:typecombine}
\end{algorithm}


\subsection{The correctness of Alg.~\ref{alg:block}}
To show that Alg.~\ref{alg:block} is correct, we need to show that the type set of all global {\qd}s of $H$ can be obtained by joining types of all local {\qd}s. 

\begin{dfn}[Restriction of {\qd}s]
Suppose $H$ is a hand and $KB$ a knowledge base. Let $b$ be a $KB$-block of $H$. For any {\qd} $\pi$ of $H$, the restriction of $\pi$ to $b$, written $\pi_{\downarrow b}$, is the subset of $\pi$ that includes all melds and pmelds of $\pi$ that are contained in $b$. 
\end{dfn}
Any restriction of a {\qd} to a $KB$-block is also a {\qd}. On the other hand, {\qd}s on different $KB$-blocks can be amalgamated. 
\begin{dfn}[Amalgamation of local {\qd}s]
Suppose $H$ is a hand and $KB$ a knowledge base. Let $\mathcal{B}$ be the set of $KB$-blocks of $H$. For each  $b\in\mathcal{B}$, let $\pi_b$ be a local {\qd} on $b$. The amalgamation $\amalg_{b\in\mathcal{B}}\ \pi_b$ of all $\pi_b$ is the union of all $\pi_b$. 
\end{dfn}
The amalgamation of all $\pi_b$ is a {\qd} on $H$ if (i) it has at most five elements and at most of one of which is an incompletable pair; and (ii) it has at least one pair when it has five elements. 

The following lemma shows that any {\qd} of a hand $H$ can be recovered from its restrictions to the $KB$-blocks of $H$. 
\begin{lemma}
Let $\mathcal{B}$ be the $KB$-blocks of $H$. Suppose  $\pi$ is a {\qd} of $H$. For any $b\in \mathcal{B}$, let $\pi_{\downarrow b}$ be the restriction of $\pi$ to $b$. Then $\pi_{\downarrow b}$ is a {\qd} of $b$. Moreover, $\pi$ is exactly the amalgamation of all $\pi_{\downarrow b}$, i.e.,  $\pi = \amalg_{b\in\mathcal{B}}\ \pi_{\downarrow b}$. 
\end{lemma}

In addition, the type of a global {\qd} is also uniquely determined by the types of its corresponding local {\qd}s. 

\begin{lemma}\label{lem:typejoin}
Suppose $H$ is a hand and $\pi$ a {\qd} on $H$. Let 
$\vec{x}=(m,n,p,e$, $re,rm,em)$ be the type of $\pi$. For each $KB$-block $b$ of $H$, let $\vec{x}_b =(m_b,n_b,p_b,e_b,re_b$, $rm_b,em_b)$ be  the type of $\pi_{\downarrow b}$, the restriction of $\pi$ to $b$. Then we have
\begin{itemize}
    \item $em=1$ iff $e=0$ and there exists $b\in \mathcal{B}$ s.t. $em_b=1$ and $re_{b'}=rm_{b'}=em_{b'}=0$ for any $b'\not=b$.
    \item  $\vec{x}[i]=\sum_b \vec{x}_b[i]$ for $0\leq i \leq 3$ and 
$\vec{x}[i] = \max_b \vec{x}_b[i]$ for $4\leq i\leq 5$. 
\end{itemize}
\end{lemma}

\begin{myexpcont}
Joining the four local {\qd}s $\pi_i$ $(1\leq i\leq 4)$ we have a {\qd} for the hand $H$: $\pi=(C6C7C8)(D1D2D3)(D6D7D8)$, with remainder $(C1C4C9D6)$. Let $\sigma(\pi_i) = (m_i,n_i,p_i,e_i,re_i,rm_i,em_i)$ be the type of $\pi_i$ for $1\leq i\leq 5$. These  are $(0,0,0,0,0,0,0)$, $(0,0,0,0,0,0,0)$, $(1,0,0,0,0,0,0)$, $(1,0,0,0,0,0,0)$ and $(1,0,0,0,1,1,1)$. Thus, the joint {\qd} has type   $(3,0,0,0$, $1,1,1)$, which is identical to the result obtained by using Lemma~\ref{lem:typejoin}.
\end{myexpcont}

The correctness of Alg.~\ref{alg:block} then follows directly from the above lemmas.

\begin{theorem}
Suppose $H$ is a hand and $KB$ a knowledge base. Let $block\_dfncy$ be the number returned by Alg.~\ref{alg:block}. Then $block\_dfncy \leq \dfncy(H,KB)$. Moreover, if $\dfncy(H,KB) \leq 4$, then  $block\_dfncy = \dfncy(H,KB)$.
\end{theorem}

\begin{remark}
In above we assume that the hand is a 13- or 14-tile. In practice, a player may also have $0\leq k\leq 4$ consolidated melds (chows, pongs, or kongs) and her hand is then a $(13-3k)$- or $(14-3k)$-tile. When calculating the type of any \qd\ of $H$, we need to increase its $m$-index by $k$. The calculation procedure is similar to the cases when $H$ is a 13- or 14-tile.
\end{remark}



%
\section{Experiments}\label{sec:experiment}

To demonstrate the efficiency of the proposed block deficiency algorithm, we compared it with the quadtree algorithm proposed in \cite{li2019lets} on randomly generated examples of hands and knowledge bases. As mentioned in the end of Sec.~\ref{sec:kb&qd}, the quadtree algorithm is incomplete when knowledge base is considered. To remedy this, we refined the quadtree algorithm by expanding the alphabet $\Sigma=\{1,2,3,4\}$ to $\{1,2,3,4,5,6\}$ and allow each node has up to six child nodes. We implemented both algorithms in Python3 and all experiments were executed on a laptop with Intel i7-6700 CPU and 16 GB RAM.

The colour set of a hand is the set of colours that tiles in the hand have. For example, the colour set of $H=(B1B1B1B3B5B6B8)(D1D2D3D5D5D5)$ is $\{B,D\}$. 
We say a hand is a pure hand if its colour set is a singleton. According to the number of different colours in the hand, we generate three experiment datasets, each contains 1000 random hands, and for each hand $H$, we randomly generate 30 or 100 knowledge bases such that they are compatible with $H$ (i.e., no tile has five or more identical copies) and their sizes (i.e., number of tiles in the knowledge base) follow a normal distribution. 



\subsection{Pure hands}
When examining pure hands, only tiles with the same colour in the knowledge base are concerned. As there are only 36 tiles with a given colour, they are at most 22 tiles in the knowledge base. Following a normal distribution, we generated 30 knowledge bases for each of the 1000 randomly  generated pure hand. The experiments are summarised in Table~\ref{tab:1}, where column 1 gives the range, say, [5,10), in which the number of tiles in KB is,  column 2 shows the number of hand-KB pairs whose KB has size located in the range specified by column 1. Note we use ms (millisecond) as the base time unit. Columns 3-6 show the maximum or average time (in ms) of the quadtree and block deficiency algorithms. The last column shows the ratio of the average time of the quadtree algorithm to that of the block deficiency algorithm. 

\begin{table}
\begin{center}
\caption{\label{tab:1} Experimental results on  pure hands}
\scalebox{0.75}{
\begin{tabular}{|c| c| c| c| c| c| c| c} \hline
\begin{tabular}{c} 
num.\\ tiles\\ in KB
\end{tabular}
& 
\begin{tabular}{c} 
num.\\ (hand,KB)\\ pairs
\end{tabular}
& 
\begin{tabular}{c} 
quadtree\\ max. calc. \\
time (ms)
\end{tabular} 
& 
\begin{tabular}{c} 
block\\ max. calc. \\ time (ms)
\end{tabular} 
&
\begin{tabular}{c} 
quadtree\\ avg. calc. \\ time (ms)
\end{tabular} 
&
\begin{tabular}{c} 
block\\ avg. calc. \\ time (ms)
\end{tabular} 
& avg. ratio\\

\hline   [0, 5)  &  4767    & $1279.2$ & $78.2$ & $86.3$ & $13.1$  & $6.6$ \\
\hline   [5, 10)  &  8923  & $5133.3$ & $110.6$ & $604.8$ & $29.6$  & $20.4$ \\
\hline   [10, 15)  &  9635  & $7280.3$ & $112.7$ & $1319.3$ & $40.1$  & $32.9$  \\
\hline   [15, 22]  &  6675  & $9344.9$ & $115.8$ & $2197.5$ & $44.8$  & $49.0$   \\
\hline [0,22]&   30,000    & $9344.9$ & $115.8$     & $1106.3$ & $33.7$  & $32.8$  \\
\hline
\end{tabular}
}
\end{center}
\end{table}

\subsection{Hands with two colours}
When examining  hands with two colours, only tiles with one of the two colours in the knowledge base are concerned. As there are only 72 tiles with one of the two given colours, they are at most 58 tiles in the knowledge base. Following a normal distribution, we generated 100 knowledge bases for each of the 1000 randomly generated 2-colour hand. The experiments are summarised in Table~\ref{tab:2}, where the meanings of the columns are the same as those in  Table~\ref{tab:1}.

\begin{table}
\begin{center}
\caption{\label{tab:2} Experimental results of  2-colour hands }
\scalebox{0.75}{
\begin{tabular}{|c| c| c| c| c| c| c| c} \hline
\begin{tabular}{c} 
num.\\ tiles\\ in KB
\end{tabular}
& 
\begin{tabular}{c} 
num.\\ (hand,KB)\\ pairs
\end{tabular}
& 
\begin{tabular}{c} 
quadtree\\ max. calc. \\
time (ms)
\end{tabular} 
& 
\begin{tabular}{c} 
block\\ max. calc. \\ time (ms)
\end{tabular} 
&
\begin{tabular}{c} 
quadtree\\ avg. calc. \\ time (ms)
\end{tabular} 
&
\begin{tabular}{c} 
block\\ avg. calc. \\ time (ms)
\end{tabular} 
& avg. ratio\\

\hline   [0, 10)  &  9963    & $424.5$ & $9.4$ & $15.1$ & $0.5$  & $28.3$ \\
\hline   [10, 20)  &  18798  & $1155.2$ & $18.0$ & $69.0$ & $0.9$  & $75.4$ \\
\hline   [20, 30)  &  25350  & $1633.0$ & $18.1$ & $156.1$ & $1.1$  & $145.6$  \\
\hline   [30, 40)  &  23646  & $2235.8$ & $18.5$ & $274.8$ & $1.2$  & $237.5$   \\
\hline   [40, 50)  &  15734  & $2525.1$ & $17.4$ & $381.8$ & $1.2$  & $311.7$   \\
\hline   [50, 58]  &  6509  & $2679.5$ & $17.3$ & $456.7$ & $1.2$  & $375.4$   \\
\hline [0,58]&   100,000    & $2679.5$ & $17.3$     & $208.8$ & $1.0$  & $200.4$  \\
\hline
\end{tabular}
}
\end{center}
\end{table}

\subsection{Hands with three colours}
When examining hands with three colours, all tiles in the knowledge base are concerned. As there are only 108 tiles, they are at most 94 tiles in the knowledge base. Following a normal distribution, we generated 100 knowledge bases for each of the 1000 randomly generated 3-colour hand. The experiments are summarised in Table~\ref{tab:3}, where the meanings of the columns are the same as those in  Table~\ref{tab:1}.

\begin{table}
\begin{center}
\caption{\label{tab:3} Experimental results of 3-colour hands }
\scalebox{0.75}{
\begin{tabular}{|c| c| c| c| c| c| c| c} \hline
\begin{tabular}{c} 
num.\\ tiles\\ in KB
\end{tabular}
& 
\begin{tabular}{c} 
num.\\ (hand,KB)\\ pairs
\end{tabular}
& 
\begin{tabular}{c} 
quadtree\\ max. calc. \\
time (ms)
\end{tabular} 
& 
\begin{tabular}{c} 
block\\ max. calc. \\ time (ms)
\end{tabular} 
&
\begin{tabular}{c} 
quadtree\\ avg. calc. \\ time (ms)
\end{tabular} 
&
\begin{tabular}{c} 
block\\ avg. calc. \\ time (ms)
\end{tabular} 
& avg. ratio\\

\hline   [0, 10)  &  4966    & $58.0$ & $3.7$ & $3.1$ & $0.2$  & $13.3$ \\
\hline   [10, 20)  &  8151  & $212.4$ & $3.4$ & $17.7$ & $0.3$  & $50.7$ \\
\hline   [20, 30)  &  11407  & $470.6$ & $4.0$ & $29.7$ & $0.4$  & $74.0$  \\
\hline   [30, 40)  &  14519  & $647.2$ & $6.7$ & $50.5$ & $0.4$  & $119.3$   \\
\hline   [40, 50)  &  15842  & $873.8$ & $4.4$ & $80.9$ & $0.4$  & $179.8$   \\
\hline   [50, 60)  &  15297  & $1054.8$ & $5.4$ & $110.8$ & $0.5$  & $242.1$   \\
\hline   [60, 70)  &  12611  & $1256.7$ & $6.4$ & $139.3$ & $0.5$  & $297.1$   \\
\hline   [70, 80)  &  9454  & $1325.8$ & $4.4$ & $165.8$ & $0.5$  & $349.6$   \\
\hline   [80, 94]  &  7753  & $1386.2$ & $4.3$ & $190.5$ & $0.5$  & $396.8$   \\
\hline [0,94]&   100,000    & $1386.2$ & $6.7$     & $90.1$ & $0.4$  & $209.6$  \\
\hline
\end{tabular}
}
\end{center}
\end{table}

\subsection{Summary of experiment results}
From our results we can see that the block deficiency algorithm is much more faster than the quadtree algorithm. In the worst cases, the quadtree algorithm takes 9.3 seconds to compute the deficiency, while the block deficiency algorithm takes at most 0.1 seconds. In average, the quadtree algorithm is 33x, 200x, and 210x slower than the block deficiency algorithm when the hands are pure, 2-colour, and 3-colour, respectively. While the quadtree algorithm is exact, it cannot be adopted by a Mahjong AI as in each step the Mahjong AI needs to respond within, say, 8 seconds while it often needs to call the quadtree algorithms dozens of times when selecting a tile to discard. Our experiments also confirm that the block deficiency algorithm is exact when the deficiency of the hand w.r.t. a knowledge base is less than or equal to 4, which guarantees its effectiveness as a deficiency algorithm. Moreover, the results show that almost all ($\geq 99\%)$ errors (i.e., the block deficiency is greater than the real deficiency) occur when the knowledge base contains less than or equal to 2 (out of 22), 10 (out of 58), and 20 (out of 94) tiles for pure hands, 2-colour hands, and 3-colour hands, respectively. In real games, this (error) rarely happens as when the knowledge base contains fewer tiles, the game is close to the end and the AI's hand should be close to ready (i.e., the deficiency should be close to 1). 

We also tested the two deficiency algorithms in complete games. Assume that the four players use the same set of Mahjong AIs and they play 100 randomly generated games in two tests. In the first test, all Mahjong AIs use the quadtree algorithm. In average, each game was completed in 219.4 seconds; in the second test, all Mahjong AIs use the block algorithm. In average, each game was completed in 1.0 seconds.  

\section{Conclusion}\label{sec:conclusion}
Computing how many tile changes are necessary to complete a Mahjong hand (i.e., the deficiency or shanten number) is a routine task for a Mahjong AI. There are very few published methods for this purpose and existing methods are either inefficient or not knowledge-aware. In this paper, we proposed an efficient algorithm for computing the deficiency of a Mahjong hand that respects the agent's knowledge of available tiles. The deficiency number calculated by this algorithm is always exact if the real deficiency is not larger than four. Experimental results on random hands and knowledge bases show that all errors occur when the real deficiency is larger than four and almost all ($\geq 99\%)$ errors occur when the knowledge base contains very few tiles, which rarely  occur simultaneously in practical games.  

The new deficiency algorithm can be used as a basic procedure by any Mahjong AI in all variants of Mahjong. We have incorporated this algorithm in designing Mahjong AI for Sichuan Mahjong, which in average can complete a game in  one second. Future work will exploit this advantage to generate huge game records and  design and train strong Mahjong AI.

\bibliographystyle{splncs04}
\small
\bibliography{mahjong}
\appendix
\section{Proofs}
\input{proofs}

\end{document}

%% file: pic/quadtree.tex
\tikzset{every picture/.style={line width=0.75pt}} 

\begin{tikzpicture}[x=0.75pt,y=0.75pt,yscale=-1,xscale=1]

\draw   (101.79,65.75) .. controls (102.46,54.84) and (113.3,46) .. (126,46) .. controls (138.7,46) and (148.46,54.84) .. (147.79,65.75) .. controls (147.12,76.66) and (136.28,85.5) .. (123.58,85.5) .. controls (110.88,85.5) and (101.12,76.66) .. (101.79,65.75) -- cycle ;
\draw    (131,87.5) -- (143.32,121.62) ;
\draw [shift={(144,123.5)}, rotate = 250.14] [color={rgb, 255:red, 0; green, 0; blue, 0 }  ][line width=0.75]    (10.93,-3.29) .. controls (6.95,-1.4) and (3.31,-0.3) .. (0,0) .. controls (3.31,0.3) and (6.95,1.4) .. (10.93,3.29)   ;
\draw  [dash pattern={on 0.84pt off 2.51pt}]  (143.5,82) -- (177.73,100.07) ;
\draw [shift={(179.5,101)}, rotate = 207.82] [color={rgb, 255:red, 0; green, 0; blue, 0 }  ][line width=0.75]    (10.93,-3.29) .. controls (6.95,-1.4) and (3.31,-0.3) .. (0,0) .. controls (3.31,0.3) and (6.95,1.4) .. (10.93,3.29)   ;
\draw  [dash pattern={on 0.84pt off 2.51pt}]  (150.79,69.75) -- (196.52,76.7) ;
\draw [shift={(198.5,77)}, rotate = 188.64] [color={rgb, 255:red, 0; green, 0; blue, 0 }  ][line width=0.75]    (10.93,-3.29) .. controls (6.95,-1.4) and (3.31,-0.3) .. (0,0) .. controls (3.31,0.3) and (6.95,1.4) .. (10.93,3.29)   ;
\draw   (81.82,208.26) .. controls (83.12,196.45) and (97.12,188.13) .. (113.09,189.69) .. controls (129.06,191.24) and (140.94,202.08) .. (139.64,213.89) .. controls (138.34,225.71) and (124.34,234.02) .. (108.37,232.47) .. controls (92.41,230.92) and (80.52,220.08) .. (81.82,208.26) -- cycle ;
\draw  [dash pattern={on 0.84pt off 2.51pt}]  (145.36,210.05) -- (214.58,209.76) ;
\draw [shift={(216.58,209.75)}, rotate = 539.76] [color={rgb, 255:red, 0; green, 0; blue, 0 }  ][line width=0.75]    (10.93,-3.29) .. controls (6.95,-1.4) and (3.31,-0.3) .. (0,0) .. controls (3.31,0.3) and (6.95,1.4) .. (10.93,3.29)   ;
\draw    (103.99,232.47) -- (88.78,268.66) ;
\draw [shift={(88,270.5)}, rotate = 292.8] [color={rgb, 255:red, 0; green, 0; blue, 0 }  ][line width=0.75]    (10.93,-3.29) .. controls (6.95,-1.4) and (3.31,-0.3) .. (0,0) .. controls (3.31,0.3) and (6.95,1.4) .. (10.93,3.29)   ;
\draw  [fill={rgb, 255:red, 128; green, 128; blue, 128 }  ,fill opacity=1 ] (197.43,249.52) .. controls (197.66,238.62) and (208.15,230) .. (220.85,230.27) .. controls (233.55,230.55) and (243.65,239.61) .. (243.42,250.51) .. controls (243.18,261.42) and (232.7,270.04) .. (220,269.76) .. controls (207.3,269.49) and (197.19,260.43) .. (197.43,249.52) -- cycle ;

\draw   (277.55,258.68) .. controls (276.83,247.8) and (286.53,238.3) .. (299.2,237.47) .. controls (311.88,236.63) and (322.73,244.78) .. (323.45,255.67) .. controls (324.16,266.55) and (314.47,276.05) .. (301.79,276.88) .. controls (289.12,277.71) and (278.26,269.57) .. (277.55,258.68) -- cycle ;

\draw   (67.51,292.64) .. controls (67.21,281.73) and (77.27,272.62) .. (89.97,272.28) .. controls (102.67,271.93) and (113.2,280.5) .. (113.49,291.4) .. controls (113.78,302.3) and (103.73,311.42) .. (91.03,311.76) .. controls (78.33,312.1) and (67.8,303.54) .. (67.51,292.64) -- cycle ;

\draw  [color={rgb, 255:red, 128; green, 128; blue, 128 }  ,draw opacity=1 ][fill={rgb, 255:red, 128; green, 128; blue, 128 }  ,fill opacity=1 ] (220.95,207.78) .. controls (220.84,196.88) and (229.81,187.94) .. (241,187.83) .. controls (252.18,187.71) and (261.34,196.46) .. (261.45,207.37) .. controls (261.56,218.27) and (252.59,227.21) .. (241.4,227.32) .. controls (230.22,227.44) and (221.06,218.69) .. (220.95,207.78) -- cycle ;

\draw  [dash pattern={on 0.84pt off 2.51pt}]  (134.7,227.18) -- (191.71,244.94) ;
\draw [shift={(193.62,245.54)}, rotate = 197.3] [color={rgb, 255:red, 0; green, 0; blue, 0 }  ][line width=0.75]    (10.93,-3.29) .. controls (6.95,-1.4) and (3.31,-0.3) .. (0,0) .. controls (3.31,0.3) and (6.95,1.4) .. (10.93,3.29)   ;
\draw    (112.17,303.07) -- (159.4,338.3) ;
\draw [shift={(161,339.5)}, rotate = 216.73] [color={rgb, 255:red, 0; green, 0; blue, 0 }  ][line width=0.75]    (10.93,-3.29) .. controls (6.95,-1.4) and (3.31,-0.3) .. (0,0) .. controls (3.31,0.3) and (6.95,1.4) .. (10.93,3.29)   ;
\draw    (83,313.5) -- (71.92,348.89) ;
\draw [shift={(71.33,350.8)}, rotate = 287.38] [color={rgb, 255:red, 0; green, 0; blue, 0 }  ][line width=0.75]    (10.93,-3.29) .. controls (6.95,-1.4) and (3.31,-0.3) .. (0,0) .. controls (3.31,0.3) and (6.95,1.4) .. (10.93,3.29)   ;
\draw    (68.75,306.32) -- (40.68,324.42) ;
\draw [shift={(39,325.5)}, rotate = 327.19] [color={rgb, 255:red, 0; green, 0; blue, 0 }  ][line width=0.75]    (10.93,-3.29) .. controls (6.95,-1.4) and (3.31,-0.3) .. (0,0) .. controls (3.31,0.3) and (6.95,1.4) .. (10.93,3.29)   ;
\draw    (102.76,313.61) -- (120.07,357.42) ;
\draw [shift={(120.81,359.28)}, rotate = 248.43] [color={rgb, 255:red, 0; green, 0; blue, 0 }  ][line width=0.75]    (10.93,-3.29) .. controls (6.95,-1.4) and (3.31,-0.3) .. (0,0) .. controls (3.31,0.3) and (6.95,1.4) .. (10.93,3.29)   ;
\draw  [color={rgb, 255:red, 128; green, 128; blue, 128 }  ,draw opacity=1 ][fill={rgb, 255:red, 155; green, 155; blue, 155 }  ,fill opacity=1 ] (201.5,83.75) .. controls (201.5,72.84) and (210.57,64) .. (221.75,64) .. controls (232.93,64) and (242,72.84) .. (242,83.75) .. controls (242,94.66) and (232.93,103.5) .. (221.75,103.5) .. controls (210.57,103.5) and (201.5,94.66) .. (201.5,83.75) -- cycle ;
\draw  [color={rgb, 255:red, 128; green, 128; blue, 128 }  ,draw opacity=1 ][fill={rgb, 255:red, 155; green, 155; blue, 155 }  ,fill opacity=1 ] (177.5,118.75) .. controls (177.5,107.84) and (186.57,99) .. (197.75,99) .. controls (208.93,99) and (218,107.84) .. (218,118.75) .. controls (218,129.66) and (208.93,138.5) .. (197.75,138.5) .. controls (186.57,138.5) and (177.5,129.66) .. (177.5,118.75) -- cycle ;
\draw  [color={rgb, 255:red, 128; green, 128; blue, 128 }  ,draw opacity=1 ][fill={rgb, 255:red, 155; green, 155; blue, 155 }  ,fill opacity=1 ] (132.5,139.75) .. controls (132.5,128.84) and (141.57,120) .. (152.75,120) .. controls (163.93,120) and (173,128.84) .. (173,139.75) .. controls (173,150.66) and (163.93,159.5) .. (152.75,159.5) .. controls (141.57,159.5) and (132.5,150.66) .. (132.5,139.75) -- cycle ;

\draw   (3.51,346.64) .. controls (3.21,335.73) and (13.27,326.62) .. (25.97,326.28) .. controls (38.67,325.93) and (49.2,334.5) .. (49.49,345.4) .. controls (49.78,356.3) and (39.73,365.42) .. (27.03,365.76) .. controls (14.33,366.1) and (3.8,357.54) .. (3.51,346.64) -- cycle ;

\draw   (40.35,375.64) .. controls (40.05,364.73) and (50.11,355.62) .. (62.81,355.28) .. controls (75.51,354.93) and (86.04,363.5) .. (86.33,374.4) .. controls (86.62,385.3) and (76.57,394.42) .. (63.87,394.76) .. controls (51.17,395.1) and (40.64,386.54) .. (40.35,375.64) -- cycle ;
\draw   (98.35,379.64) .. controls (98.05,368.73) and (108.11,359.62) .. (120.81,359.28) .. controls (133.51,358.93) and (144.04,367.5) .. (144.33,378.4) .. controls (144.62,389.3) and (134.57,398.42) .. (121.87,398.76) .. controls (109.17,399.1) and (98.64,390.54) .. (98.35,379.64) -- cycle ;
\draw   (152.35,358.64) .. controls (152.05,347.73) and (162.11,338.62) .. (174.81,338.28) .. controls (187.51,337.93) and (198.04,346.5) .. (198.33,357.4) .. controls (198.62,368.3) and (188.57,377.42) .. (175.87,377.76) .. controls (163.17,378.1) and (152.64,369.54) .. (152.35,358.64) -- cycle ;
\draw    (124,232.5) .. controls (112.12,274.08) and (237.45,299.98) .. (278.77,271.38) ;
\draw [shift={(280,270.5)}, rotate = 503.13] [color={rgb, 255:red, 0; green, 0; blue, 0 }  ][line width=0.75]    (10.93,-3.29) .. controls (6.95,-1.4) and (3.31,-0.3) .. (0,0) .. controls (3.31,0.3) and (6.95,1.4) .. (10.93,3.29)   ;
\draw    (317.6,266.48) -- (367.67,311.36) ;
\draw [shift={(369.16,312.7)}, rotate = 221.88] [color={rgb, 255:red, 0; green, 0; blue, 0 }  ][line width=0.75]    (10.93,-3.29) .. controls (6.95,-1.4) and (3.31,-0.3) .. (0,0) .. controls (3.31,0.3) and (6.95,1.4) .. (10.93,3.29)   ;
\draw    (297.4,277.87) -- (281.57,330.72) ;
\draw [shift={(281,332.63)}, rotate = 286.67] [color={rgb, 255:red, 0; green, 0; blue, 0 }  ][line width=0.75]    (10.93,-3.29) .. controls (6.95,-1.4) and (3.31,-0.3) .. (0,0) .. controls (3.31,0.3) and (6.95,1.4) .. (10.93,3.29)   ;
\draw    (287.64,277.92) -- (255.31,314.99) ;
\draw [shift={(254,316.5)}, rotate = 311.09000000000003] [color={rgb, 255:red, 0; green, 0; blue, 0 }  ][line width=0.75]    (10.93,-3.29) .. controls (6.95,-1.4) and (3.31,-0.3) .. (0,0) .. controls (3.31,0.3) and (6.95,1.4) .. (10.93,3.29)   ;
\draw    (305.53,276.69) -- (330.28,334.44) ;
\draw [shift={(331.07,336.28)}, rotate = 246.8] [color={rgb, 255:red, 0; green, 0; blue, 0 }  ][line width=0.75]    (10.93,-3.29) .. controls (6.95,-1.4) and (3.31,-0.3) .. (0,0) .. controls (3.31,0.3) and (6.95,1.4) .. (10.93,3.29)   ;
\draw   (213.1,335.06) .. controls (211.75,324.24) and (220.87,314.19) .. (233.48,312.62) .. controls (246.08,311.05) and (257.4,318.55) .. (258.74,329.38) .. controls (260.09,340.2) and (250.97,350.25) .. (238.36,351.82) .. controls (225.76,353.39) and (214.44,345.88) .. (213.1,335.06) -- cycle ;
\draw   (252.57,360.36) .. controls (251.22,349.53) and (260.35,339.48) .. (272.95,337.91) .. controls (285.56,336.34) and (296.87,343.85) .. (298.22,354.67) .. controls (299.57,365.49) and (290.44,375.54) .. (277.84,377.11) .. controls (265.23,378.68) and (253.92,371.18) .. (252.57,360.36) -- cycle ;
\draw  [fill={rgb, 255:red, 128; green, 128; blue, 128 }  ,fill opacity=1 ] (310.69,358.72) .. controls (309.34,347.89) and (318.46,337.85) .. (331.07,336.28) .. controls (343.67,334.71) and (354.99,342.21) .. (356.33,353.03) .. controls (357.68,363.86) and (348.56,373.9) .. (335.95,375.47) .. controls (323.35,377.04) and (312.03,369.54) .. (310.69,358.72) -- cycle ;
\draw  [fill={rgb, 255:red, 128; green, 128; blue, 128 }  ,fill opacity=1 ] (362.4,332.58) .. controls (361.05,321.76) and (370.18,311.71) .. (382.78,310.14) .. controls (395.39,308.57) and (406.7,316.07) .. (408.05,326.9) .. controls (409.39,337.72) and (400.27,347.77) .. (387.66,349.34) .. controls (375.06,350.91) and (363.75,343.41) .. (362.4,332.58) -- cycle ;
\draw   (56.41,124.94) .. controls (56.15,115.31) and (64.24,107.27) .. (74.47,107) .. controls (84.7,106.72) and (93.21,114.31) .. (93.47,123.94) .. controls (93.73,133.57) and (85.64,141.61) .. (75.41,141.88) .. controls (65.17,142.16) and (56.67,134.57) .. (56.41,124.94) -- cycle ;
\draw    (107,82.5) -- (89.07,110.81) ;
\draw [shift={(88,112.5)}, rotate = 302.35] [color={rgb, 255:red, 0; green, 0; blue, 0 }  ][line width=0.75]    (10.93,-3.29) .. controls (6.95,-1.4) and (3.31,-0.3) .. (0,0) .. controls (3.31,0.3) and (6.95,1.4) .. (10.93,3.29)   ;
\draw  [dash pattern={on 0.84pt off 2.51pt}]  (76.86,141.51) .. controls (80.53,155.96) and (86.15,160.98) .. (92.03,162.68) .. controls (92.27,162.75) and (92.51,162.81) .. (92.75,162.87) .. controls (95.31,163.5) and (97.87,163.59) .. (100.28,163.69) .. controls (105.85,163.91) and (110.46,164.39) .. (113.05,169.32) .. controls (114.33,171.75) and (115.11,175.37) .. (115.11,180.82) .. controls (115.11,183.38) and (114.94,186.36) .. (114.58,189.84)(73.95,142.25) .. controls (78.09,158.53) and (84.69,163.68) .. (91.2,165.56) .. controls (91.47,165.64) and (91.75,165.72) .. (92.03,165.78) .. controls (94.8,166.47) and (97.56,166.58) .. (100.16,166.69) .. controls (104.41,166.86) and (108.32,166.77) .. (110.4,170.72) .. controls (111.52,172.85) and (112.11,176.05) .. (112.11,180.82) .. controls (112.11,183.29) and (111.95,186.17) .. (111.6,189.53) ;
\draw [shift={(98.07,165.09)}, rotate = 187.68] [fill={rgb, 255:red, 0; green, 0; blue, 0 }  ][line width=0.08]  [draw opacity=0] (10.72,-5.15) -- (0,0) -- (10.72,5.15) -- (7.12,0) -- cycle    ;

\draw (231.92,198.32) node [anchor=north west][inner sep=0.75pt]  [rotate=-359.41]  {$\alpha 4$};
\draw (81.25,282.91) node [anchor=north west][inner sep=0.75pt]  [rotate=-358.46]  {$\alpha 1$};
\draw (289.84,247.5) node [anchor=north west][inner sep=0.75pt]  [rotate=-356.24]  {$\alpha 2$};
\draw (211.65,239.48) node [anchor=north west][inner sep=0.75pt]  [rotate=-1.23]  {$\alpha 3$};
\draw (321,50.4) node [anchor=north west][inner sep=0.75pt]    {$H=( B1B4B7)( C2C5C8)( D1D4D7D8D8D9D9D9)$};
\draw (478.5,91.4) node [anchor=north west][inner sep=0.75pt]  [font=\fontsize{1em}{1.2em}\selectfont]  {$S_{\alpha } =( D7D8D8D9D9D9)$};
\draw (335,90) node [anchor=north west][inner sep=0.75pt]  [font=\fontsize{1.5em}{1.7em}\selectfont]  {$\alpha =1^{8}$};
\draw (423,254.4) node [anchor=north west][inner sep=0.75pt]  [font=\fontsize{1.1em}{1.3em}\selectfont]  {$\pi _{\alpha } =\pi _{\varepsilon } =\pi _{\alpha 1} =\pi _{\alpha 11} =( \cdot \cdot )( \cdot \cdot )( \cdot \cdot )( \cdot \cdot )( \cdot )$};
\draw (408,363.4) node [anchor=north west][inner sep=0.75pt]  [font=\fontsize{1em}{1.2em}\selectfont]  {$\pi _{\alpha 2} =\pi _{\alpha 21} =( D7D8D9)( \cdot \cdot )( \cdot \cdot )( \cdot \cdot )( \cdot )$};
\draw (447,381.4) node [anchor=north west][inner sep=0.75pt]  [font=\fontsize{1em}{1.2em}\selectfont]  {$\pi _{\alpha 22} =( D7D8D9)( D8D9)( \cdot \cdot )( \cdot \cdot )( \cdot )$};
\draw (110.73,209.08) node  [rotate=-1.4]  {$\alpha =1^{8}$};
\draw (170.14,191.33) node [anchor=north west][inner sep=0.75pt]  [rotate=-343.87]  {$4$};
\draw (159.92,217.18) node [anchor=north west][inner sep=0.75pt]  [rotate=-343.87]  {$3$};
\draw (158.4,253.96) node [anchor=north west][inner sep=0.75pt]  [rotate=-11]  {$2$};
\draw (84.14,237.31) node [anchor=north west][inner sep=0.75pt]  [rotate=-2.37]  {$1$};
\draw (122.1,62.75) node    {$\varepsilon $};
\draw (170.51,56.4) node [anchor=north west][inner sep=0.75pt]    {$4$};
\draw (160.1,74.4) node [anchor=north west][inner sep=0.75pt]    {$3$};
\draw (140.87,92.4) node [anchor=north west][inner sep=0.75pt]    {$2$};
\draw (83.57,80.4) node [anchor=north west][inner sep=0.75pt]    {$1$};
\draw (215.57,74.4) node [anchor=north west][inner sep=0.75pt]    {$4$};
\draw (192.57,111.4) node [anchor=north west][inner sep=0.75pt]    {$3$};
\draw (47.09,365.91) node [anchor=north west][inner sep=0.75pt]  [rotate=-358.46]  {$\alpha 12$};
\draw (106.09,369.91) node [anchor=north west][inner sep=0.75pt]  [rotate=-358.46]  {$\alpha 13$};
\draw (159.09,348.91) node [anchor=north west][inner sep=0.75pt]  [rotate=-358.46]  {$\alpha 14$};
\draw (134.82,303.67) node [anchor=north west][inner sep=0.75pt]  [rotate=-3.92]  {$4$};
\draw (115.33,323.68) node [anchor=north west][inner sep=0.75pt]  [rotate=-3.92]  {$3$};
\draw (84.6,333.39) node [anchor=north west][inner sep=0.75pt]  [rotate=-3.92]  {$2$};
\draw (41.86,298.38) node [anchor=north west][inner sep=0.75pt]  [rotate=-2.1]  {$1$};
\draw (10.25,337.91) node [anchor=north west][inner sep=0.75pt]  [rotate=-358.46]  {$\alpha 11$};
\draw (473.5,110.4) node [anchor=north west][inner sep=0.75pt]  [font=\fontsize{1em}{1.2em}\selectfont]  {$S_{\alpha 1} =( D8D8D9D9D9)$};
\draw (474.5,177.4) node [anchor=north west][inner sep=0.75pt]  [font=\fontsize{1em}{1.2em}\selectfont]  {$S_{\alpha 2} =( D8D9D9)$};
\draw (259.45,347.66) node [anchor=north west][inner sep=0.75pt]  [rotate=-2.05]  {$\alpha 22$};
\draw (320.02,348.13) node [anchor=north west][inner sep=0.75pt]  [rotate=-358.96]  {$\alpha 23$};
\draw (372.43,319.89) node [anchor=north west][inner sep=0.75pt]  [rotate=-3.33]  {$\alpha 24$};
\draw (322.17,301.38) node [anchor=north west][inner sep=0.75pt]  [rotate=-358.36]  {$3$};
\draw (292.53,314.02) node [anchor=north west][inner sep=0.75pt]  [rotate=-358.36]  {$2$};
\draw (339.71,268.51) node [anchor=north west][inner sep=0.75pt]  [rotate=-358.36]  {$4$};
\draw (220.07,323.37) node [anchor=north west][inner sep=0.75pt]  [rotate=-2]  {$\alpha 21$};
\draw (470.5,210.4) node [anchor=north west][inner sep=0.75pt]  [font=\fontsize{1em}{1.2em}\selectfont]  {$S_{\alpha 22} =( D9)$};
\draw (469.5,193.4) node [anchor=north west][inner sep=0.75pt]  [font=\fontsize{1em}{1.2em}\selectfont]  {$S_{\alpha 21} =( D9D9)$};
\draw (447,297.4) node [anchor=north west][inner sep=0.75pt]  [font=\fontsize{1em}{1.2em}\selectfont]  {$\pi _{\alpha 12} =( D8D9)( \cdot \cdot )( \cdot \cdot )( \cdot \cdot )( \cdot )$};
\draw (448,315.4) node [anchor=north west][inner sep=0.75pt]  [font=\fontsize{1em}{1.2em}\selectfont]  {$\pi _{\alpha 13} =( \cdot \cdot )( \cdot \cdot )( \cdot \cdot )( \cdot \cdot )( D8D8)$};
\draw (447,333.4) node [anchor=north west][inner sep=0.75pt]  [font=\fontsize{1em}{1.2em}\selectfont]  {$\pi _{\alpha 14} =( D8D8)( \cdot \cdot )( \cdot \cdot )( \cdot \cdot )( \cdot )$};
\draw (145.5,131.4) node [anchor=north west][inner sep=0.75pt]    {$2$};
\draw (469.5,125.4) node [anchor=north west][inner sep=0.75pt]  [font=\fontsize{1em}{1.2em}\selectfont]  {$S_{\alpha 11} =( D8D9D9D9)$};
\draw (470.5,142.4) node [anchor=north west][inner sep=0.75pt]  [font=\fontsize{1em}{1.2em}\selectfont]  {$S_{\alpha 12} =( D8D9D9)$};
\draw (428.5,159.4) node [anchor=north west][inner sep=0.75pt]  [font=\fontsize{1em}{1.2em}\selectfont]  {$S_{\alpha 13} =S_{\alpha 14} =( D9D9D9)$};
\draw (255.86,286.38) node [anchor=north west][inner sep=0.75pt]  [rotate=-2.1]  {$1$};
\draw (68.09,115.91) node [anchor=north west][inner sep=0.75pt]  [rotate=-358.46]  {$1$};

\end{tikzpicture}

%% file: pic/block2.tex
\tikzset{every picture/.style={line width=0.75pt}} 

\begin{tikzpicture}[x=0.75pt,y=0.75pt,yscale=-1,xscale=1]

\draw   (71,63) -- (141,63) -- (141,106) -- (71,106) -- cycle ;
\draw   (70,127) -- (140,127) -- (140,170) -- (70,170) -- cycle ;
\draw   (71,290) -- (141,290) -- (141,333) -- (71,333) -- cycle ;
\draw   (152,70) -- (182,85) -- (152,100) -- (167,85) -- cycle ;
\draw   (153,132) -- (183,147) -- (153,162) -- (168,147) -- cycle ;
\draw   (154,298) -- (184,313) -- (154,328) -- (169,313) -- cycle ;
\draw   (313,70) -- (343,85) -- (313,100) -- (328,85) -- cycle ;
\draw   (372,59) .. controls (367.33,58.91) and (364.95,61.19) .. (364.86,65.86) -- (364.67,75.32) .. controls (364.54,81.98) and (362.14,85.26) .. (357.47,85.17) .. controls (362.14,85.26) and (364.4,88.64) .. (364.27,95.31)(364.33,92.31) -- (364.14,101.86) .. controls (364.05,106.53) and (366.33,108.91) .. (371,109) ;
\draw   (417,109) .. controls (421.67,109) and (424,106.67) .. (424,102) -- (424,95) .. controls (424,88.33) and (426.33,85) .. (431,85) .. controls (426.33,85) and (424,81.67) .. (424,75)(424,78) -- (424,67) .. controls (424,62.33) and (421.67,60) .. (417,60) ;
\draw   (70,206) -- (140,206) -- (140,249) -- (70,249) -- cycle ;
\draw   (153,214) -- (183,229) -- (153,244) -- (168,229) -- cycle ;
\draw   (530.15,163) .. controls (525.48,162.87) and (523.09,165.14) .. (522.96,169.81) -- (522.55,185.06) .. controls (522.37,191.73) and (519.95,195) .. (515.28,194.87) .. controls (519.95,195) and (522.19,198.39) .. (522.01,205.06)(522.09,202.06) -- (521.73,215.62) .. controls (521.6,220.28) and (523.87,222.67) .. (528.53,222.8) ;
\draw   (581.96,224) .. controls (586.63,224) and (588.96,221.67) .. (588.96,217) -- (588.96,205.52) .. controls (588.96,198.85) and (591.29,195.52) .. (595.96,195.52) .. controls (591.29,195.52) and (588.96,192.19) .. (588.96,185.52)(588.96,188.52) -- (588.96,172.4) .. controls (588.96,167.73) and (586.63,165.4) .. (581.96,165.4) ;
\draw  [color={rgb, 255:red, 189; green, 16; blue, 224 }  ,draw opacity=1 ] (550.5,59) -- (600,85.5) -- (550.5,112) -- (501,85.5) -- cycle ;

\draw   (539.31,151.55) -- (539.52,136.3) -- (528.29,136.15) -- (551.04,114.84) -- (573.2,136.75) -- (561.98,136.6) -- (561.77,151.85) -- cycle ;
\draw   (312,133) -- (342,148) -- (312,163) -- (327,148) -- cycle ;
\draw   (371,122) .. controls (366.33,121.91) and (363.95,124.19) .. (363.86,128.86) -- (363.67,138.32) .. controls (363.54,144.98) and (361.14,148.26) .. (356.47,148.17) .. controls (361.14,148.26) and (363.4,151.64) .. (363.27,158.31)(363.33,155.31) -- (363.14,164.86) .. controls (363.05,169.53) and (365.33,171.91) .. (370,172) ;
\draw   (418,172) .. controls (422.67,172) and (425,169.67) .. (425,165) -- (425,158) .. controls (425,151.33) and (427.33,148) .. (432,148) .. controls (427.33,148) and (425,144.67) .. (425,138)(425,141) -- (425,130) .. controls (425,125.33) and (422.67,123) .. (418,123) ;
\draw   (315,204) -- (345,219) -- (315,234) -- (330,219) -- cycle ;
\draw   (374,193) .. controls (369.33,192.91) and (366.95,195.19) .. (366.86,199.86) -- (366.67,209.32) .. controls (366.54,215.98) and (364.14,219.26) .. (359.47,219.17) .. controls (364.14,219.26) and (366.4,222.64) .. (366.27,229.31)(366.33,226.31) -- (366.14,235.86) .. controls (366.05,240.53) and (368.33,242.91) .. (373,243) ;
\draw   (419,243) .. controls (423.67,243) and (426,240.67) .. (426,236) -- (426,229) .. controls (426,222.33) and (428.33,219) .. (433,219) .. controls (428.33,219) and (426,215.67) .. (426,209)(426,212) -- (426,201) .. controls (426,196.33) and (423.67,194) .. (419,194) ;
\draw   (315,298) -- (345,313) -- (315,328) -- (330,313) -- cycle ;
\draw   (374,287) .. controls (369.33,286.91) and (366.95,289.19) .. (366.86,293.86) -- (366.67,303.32) .. controls (366.54,309.98) and (364.14,313.26) .. (359.47,313.17) .. controls (364.14,313.26) and (366.4,316.64) .. (366.27,323.31)(366.33,320.31) -- (366.14,329.86) .. controls (366.05,334.53) and (368.33,336.91) .. (373,337) ;
\draw   (417,336) .. controls (421.67,336) and (424,333.67) .. (424,329) -- (424,322) .. controls (424,315.33) and (426.33,312) .. (431,312) .. controls (426.33,312) and (424,308.67) .. (424,302)(424,305) -- (424,294) .. controls (424,289.33) and (421.67,287) .. (417,287) ;
\draw  [color={rgb, 255:red, 248; green, 231; blue, 28 }  ,draw opacity=1 ] (351,44) -- (444,44) -- (444,167.5) -- (463.7,167.5) -- (463.7,132.98) -- (512,196) -- (463.7,259.02) -- (463.7,224.5) -- (444,224.5) -- (444,348) -- (351,348) -- cycle ;

\draw (532,78) node [anchor=north west][inner sep=0.75pt]   [align=left] {decide};
\draw (101,172) node [anchor=north west][inner sep=0.75pt]    {$\vdots $};
\draw (84,365) node [anchor=north west][inner sep=0.75pt]   [align=left] {$KB$-blocks};
\draw (187,68.4) node [anchor=north west][inner sep=0.75pt]    {$\Bigl\{{\qd}s\ of\ b_{1}\Bigr\}$};
\draw (100,79.4) node [anchor=north west][inner sep=0.75pt]    {$b_{1}$};
\draw (190,131.4) node [anchor=north west][inner sep=0.75pt]    {$\Bigl\{{\qd}s\ of\ b_{2}\Bigr\}$};
\draw (190,296.4) node [anchor=north west][inner sep=0.75pt]    {$\Bigl\{{\qd}s\ of\ b_{\ell }\Bigr\}$};
\draw (96,136.4) node [anchor=north west][inner sep=0.75pt]    {$b_{2}$};
\draw (98,302.4) node [anchor=north west][inner sep=0.75pt]    {$b_{\ell }$};
\draw (161,172) node [anchor=north west][inner sep=0.75pt]    {$\vdots $};
\draw (239,172) node [anchor=north west][inner sep=0.75pt]    {$\vdots $};
\draw (190,217.4) node [anchor=north west][inner sep=0.75pt]    {$\Bigl\{{\qd}s\ of\ b_{s}\Bigr\}$};

\draw (321,172) node [anchor=north west][inner sep=0.75pt]    {$\vdots $};
\draw (390,172) node [anchor=north west][inner sep=0.75pt]    {$\vdots $};
\draw (97,218.4) node [anchor=north west][inner sep=0.75pt]    {$b_{s}$};
\draw (102,255) node [anchor=north west][inner sep=0.75pt]    {$\vdots $};
\draw (158,255) node [anchor=north west][inner sep=0.75pt]    {$\vdots $};
\draw (235,255) node [anchor=north west][inner sep=0.75pt]    {$\vdots $};
\draw (193,356) node [anchor=north west][inner sep=0.75pt]   [align=left] {\begin{minipage}[lt]{42.98pt}\setlength\topsep{0pt}
\begin{center}
{local\\\qd}s
\end{center}

\end{minipage}};
\draw (356,354) node [anchor=north west][inner sep=0.75pt]   [align=left] {\begin{minipage}[lt]{62.25pt}\setlength\topsep{0pt}
\begin{center}
types of \\local {\qd}s
\end{center}

\end{minipage}};
\draw (533,354) node [anchor=north west][inner sep=0.75pt]   [align=left] {\begin{minipage}[lt]{29.94pt}\setlength\topsep{0pt}
\begin{center}
types of $\displaystyle T$
\end{center}

\end{minipage}};
\draw (378,65.4) node [anchor=north west][inner sep=0.75pt]    {$ \begin{array}{l}
types\ \\
of\ b_{1}
\end{array}$};
\draw (378,128.4) node [anchor=north west][inner sep=0.75pt]    {$ \begin{array}{l}
types\ \\
of\ b_{2}
\end{array}$};
\draw (378,200.4) node [anchor=north west][inner sep=0.75pt]    {$ \begin{array}{l}
types\ \\
of\ b_{s}
\end{array}$};
\draw (379,293.4) node [anchor=north west][inner sep=0.75pt]    {$ \begin{array}{l}
types\ \\
of\ b_{\ell }
\end{array}$};
\draw (321,255) node [anchor=north west][inner sep=0.75pt]    {$\vdots $};
\draw (389,255) node [anchor=north west][inner sep=0.75pt]    {$\vdots $};
\draw (536.67,176.5) node [anchor=north west][inner sep=0.75pt]    {$ \begin{array}{l}
{\displaystyle \mathnormal{types\ }}\\
{\displaystyle \mathnormal{\ of\ T}}
\end{array}$};
\draw (458,187) node [anchor=north west][inner sep=0.75pt]   [align=left] {JOIN};

\end{tikzpicture}

%% file: proofs.tex
\begin{proof}[Proof of Lemma~\ref{lem:incompletable}]
(a) If $p=re=ke=0$, $\pi$ has no pair, and we cannot make a pair by using tiles in the remainder of $\pi$ or the knowledge base. Thus $\pi$ is incompletable. 

(b) If $re=ke=0$, then the only way we can have the eye is to use one pair in $\pi$ to act as the eye. This is impossible if $p=0$. In case $p>0$, we need to create a new meld. However, by $rm=km=0$, this is not possible either. Thus $\pi$ is incompletable in this case.

(c) Suppose $m+n-e\leq 3$. In order to complete $\pi$, we need to create a new meld, but this is impossible if both $rm$ and $km$ are 0. 

(d) Suppose $m+n\leq 3$, $p=ke=km=0$ and $re=rm=1$. If we cannot make the eye and a meld simultaneously, i.e., $em=1$, $\pi$ is still incompletable. 
\end{proof}

\begin{proof}[Proof of Lemma~\ref{lem:m+n>=4}]
By Lemma~\ref{lem:7attr}, we know $m+n>4$ only if $p>0$. 

For case (a), we have $p>0$ and $e\leq 1$. If $e=1$, then all the other pairs in $\pi$ are completable; if $e=0$, we may use one pair in $\pi$ as the eye. In either subcase, we need only complete $4-m$ pmelds. 

For case (b), if $e=0$ and $re=1$, the eye can be created by using a tile in the remainder and a tile from the knowledge base. Thus the cost is $4-m+1$. If $p>0$ and $rm=1$, we assign an existing pair as the eye and create a meld by starting with a tile in the remainder. The cost is also $4-m+1$ as we need to complete $4-m-1$ pmelds and create a meld from scratch.

In case (c), since neither condition of case (b) holds, we cannot complete the hand by making the eye or an meld by starting with a tile in the remainder. If $e=re=0$ and $ke=1$, we can make the eye by borrowing two tiles from $KB$; if $p>0$, $rm=0$, and $km=1$, we use one pair as the eye and make a new meld by borrowing three tiles from $KB$. In either case, the cost is $4-m+2$. 

In case (d), none of the above holds. We have (i) $m+n=4$, (ii) if $e=0$ then $re=ke=0$, and (iii) if $p>0$ then $rm=km=0$. There are three subcases. If $p>e=0$, then $rm=km=re=ke=0$. By Lemma~\ref{lem:incompletable} (b), $\pi$ is incompletable. If $p=0$, then $e=0$ and hence $re=ke=0$. By Lemma~\ref{lem:incompletable} (a), $\pi$ is incompletable. If $p>0$ and $e>0$, then $rm=km=0$ and $m+n-e\leq 3$. By Lemma~\ref{lem:incompletable} (c), $\pi$ is incompletable.
\end{proof}

\begin{proof}[Proof of Lemma~\ref{lem:cost3}]
Except at most one pair, each pmeld in $\pi$ can be completed by one tile change. To complete $\pi$, we need also to create at least one meld. In addition, if $p=0$, we need also create the eye by using tiles in the remainder of $\pi$ and $KB$. In case $e=0$ but $p>0$, we may assign an existing pair in $\pi$ as the eye, but this incurs the cost of creating another meld from scratch.

Suppose we want to create a meld from scratch; that is, do not use any pmeld in $\pi$. We pick one tile from the remainder if $rm=1$ and then complete it by borrowing two tiles in $KB$. When $rm=0$, we need to create the meld by borrowing three tiles from $KB$. This shows that $mcost$ is the minimum number of tile changes for creating such a meld.

The case of creating the eye is only slightly different and depends on the values of $p,e,re$ and $ke$. If $re=1$ or $ke=1$, and we want to create the eye from scratch, then the cost is either 1 (if $re=1$) or 2 (if $re=0$ and $ke=1$). That is, the cost to make the eye in this case is $ecost$.

In case (a), as $e=1$, we already have the eye. The cost of $\pi$ is $n-e$ (the number of completable pmelds in $\pi$) plus $mcost$, the minimal cost of making a meld from scratch. 

In case (b), $p=e=0$. We have $re=1$ or $ke=1$. Thus the cost of $\pi$ is $n$ (the number of completable pmelds in $\pi$) plus $mcost$ plus $ecost$.

In case (c), $p=0$, $em=1$, and $ke=1$ or $km=1$. From $em=1$, we have $re=rm=1$ and $ecost=1$, $mcost=2$, but we cannot make the eye and a new meld simultaneously from the remainder. By $ke=1$ or $km=1$, we have the cost is  $n+mcost+ecost+1=n+4$. 

In case (d), $p>0$ and $e=0$. We may complete $\pi$ by (i) creating a new meld and the eye from scratch, or (ii) assigning an existing pair in $\pi$ as the eye and creating two new melds from scratch. We divide the discussion into subcases. As $p>0$, it is possible that both $re$ and $ke$ are 0. Here we also note that $n\geq p>0$.
\begin{itemize}
    \item If $re=1$, then $ecost=1$ and the cost of $\pi$ is $f_1 = n+mcost+1$, which is at least $4$ as $mcost\geq 2$. If the eye-meld conflict exists, the cost is at least 5. 
    
    \item If $re=ke=0$, then we can only make the eye by using an existing pair in $\pi$. Thus the cost of $\pi$ is at least $f_2 = n-1+2\times mcost\geq 4$, as $mcost\geq 2$. 
    \item Suppose $re=0$, $ke=1$ and  $rm=0$. Since $\max(rm,km)=1$, we have $km=1$ and, thus, $mcost=3$, $ecost=2$. Creating two melds from scratch requires at least 6 tile changes and thus it is cheaper to create the eye from scratch. In this case, the cost of $\pi$ is at least $n+mcost+ecost=n+5\geq 6$.
    \item Suppose $re=0$, $ke=1$ and $rm=1$. Then $mcost=2$ and $ecost=2$. If we choose to create the eye from scratch, then the cost is $n+mcost+ecost=n+4$. Instead, suppose we assign one existing pair as the eye and make two melds from scratch. This incurs cost at least $n-1+2\times mcost=n+3$. As a result, the cost of $\pi$ in this case is at least $n+3\geq 4$.
\end{itemize}
It is then clear that the cost $\pi$ in case (d) is at least the minimum of $f_1$ and $f_2$ and thus at least 4.
\end{proof}